\documentclass{article} 
\usepackage[numbers]{natbib}
\usepackage[final]{nips_2016_arxiv}
\usepackage{times}
\usepackage{url}
\usepackage{amsmath}
\usepackage{amssymb}
\usepackage{amsthm}
\usepackage{bigints}
\usepackage{algorithm}
\usepackage{verbatim}
\usepackage{algpseudocode}
\usepackage{rotating}
\usepackage{subfigure}
\usepackage{caption}
\usepackage{microtype}
\usepackage{wrapfig}
\usepackage{float}
\usepackage{paralist}	
\usepackage{booktabs}

\usepackage{appendix}	

\usepackage[labelfont=bf]{caption}

\usepackage{hyperref}

\usepackage[usenames,svgnames]{xcolor}

\newcommand{\varf}{V_{\omega}}

\newcommand{\inner}[2]{\left\langle #1,#2\right\rangle}
\newtheorem{theorem}{Theorem}

\title{$f$-GAN: Training Generative Neural Samplers using Variational
Divergence Minimization}

\author{
Sebastian Nowozin\\
Machine Intelligence and Perception Group\\
Microsoft Research\\
Cambridge, UK\\
\texttt{Sebastian.Nowozin@microsoft.com}\\
\And
Botond Cseke\\
Machine Intelligence and Perception Group\\
Microsoft Research\\
Cambridge, UK\\
\texttt{botcse@microsoft.com}
\And
Ryota Tomioka\\
Machine Intelligence and Perception Group\\
Microsoft Research\\
Cambridge, UK\\
\texttt{ryoto@microsoft.com}
}

\begin{document}

\maketitle

\begin{abstract}
Generative neural samplers are probabilistic models that implement sampling
using feedforward neural networks: they take a random input vector and produce
a sample from a probability distribution defined by the network weights.
These models are expressive and allow efficient computation of samples and
derivatives, but cannot be used for computing likelihoods or for marginalization.
The \emph{generative-adversarial} training method allows to train such models
through the use of an auxiliary discriminative neural network.
We show that the generative-adversarial approach is a special case of an
existing more general variational divergence estimation approach.
We show that any $f$-divergence can be used
for training generative neural samplers.
We discuss the benefits of various choices of divergence functions on training
complexity and the quality of the obtained generative models.
\end{abstract}

\section{Introduction}
Probabilistic generative models describe a probability distribution over a
given domain $\mathcal{X}$, for example a distribution over natural language
sentences, natural images, or recorded waveforms.

Given a generative model $Q$ from a class $\mathcal{Q}$ of possible models we
are generally interested in performing one or multiple of the following
operations:
\begin{compactitem}
\item \emph{Sampling.}  Produce a sample from $Q$.  By inspecting samples or
calculating a function on a set of samples we can obtain important insight
into the distribution or solve decision problems.
\item \emph{Estimation.}  Given a set of iid samples $\{x_1,x_2,\dots,x_n\}$ from
an unknown true distribution $P$, find $Q \in \mathcal{Q}$ that best describes
the true distribution.
\item \emph{Point-wise likelihood evaluation.}  Given a sample $x$, evaluate
the likelihood $Q(x)$.
\end{compactitem}

\emph{Generative-adversarial networks} (GAN) in the form proposed
by~\cite{goodfellow2014generativeadversarial} are an expressive class of
generative models that allow exact sampling and approximate estimation.
The model used in GAN is simply a feedforward neural network which receives as input a
vector of random numbers, sampled, for example, from a uniform distribution.
This random input is passed through each layer in the network and the final
layer produces the desired output, for example, an image.
Clearly, sampling from a GAN model is efficient because only one forward pass
through the network is needed to produce one exact sample.

Such probabilistic feedforward neural network models were first considered
in~\cite{mackay1995bayesianneuralnetworks} and~\cite{bishop1998gtm},
here we call these models \textbf{generative neural samplers}.
GAN is also of this type, as is the decoder model of a variational
autoencoder~\cite{kingma2013vae}.

In the original GAN paper the authors show that it is possible to estimate
neural samplers by approximate minimization of the symmetric
\emph{Jensen-Shannon divergence},
\begin{equation}
D_{\textrm{JS}}(P\|Q) = \tfrac{1}{2} D_{\textrm{KL}}(P\|\tfrac{1}{2}(P+Q))
	+ \tfrac{1}{2} D_{\textrm{KL}}(Q\| \tfrac{1}{2}(P+Q)),
\end{equation}
where $D_{KL}$ denotes the Kullback-Leibler divergence.
The key technique used in the GAN training is that of introducing a second
``\emph{discriminator}'' neural networks which is optimized simultaneously.
Because $D_{\textrm{JS}}(P\|Q)$ is a proper divergence measure between
distributions this implies that the true distribution $P$ can be approximated
well in case there are sufficient training samples and the model class
$\mathcal{Q}$ is rich enough to represent $P$.

In this work we show that the principle of GANs is more general and we can
extend the variational divergence estimation framework proposed by Nguyen
et al.~\cite{nguyen2010divergenceestimation} to recover the GAN training
objective and generalize it to arbitrary $f$-divergences.

More concretely, we make the following contributions over the state-of-the-art:
\begin{compactitem}
\item We derive the GAN training objectives for all $f$-divergences and
provide as example additional divergence functions, including the
Kullback-Leibler and Pearson divergences.
\item We simplify the saddle-point optimization procedure of Goodfellow et
al.~\cite{goodfellow2014generativeadversarial} and provide a theoretical
justification.
\item We provide experimental insight into which divergence function is
suitable for estimating generative neural samplers for natural images.
\end{compactitem}

\section{Method}
We first review the divergence estimation framework of Nguyen et
al.~\cite{nguyen2010divergenceestimation} which is based on $f$-divergences.
We then extend this framework from divergence estimation to model estimation.

\subsection{The f-divergence Family}
Statistical divergences such as the well-known \emph{Kullback-Leibler
divergence} measure the difference between two given probability
distributions.
A large class of different divergences are the so called
$f$-divergences~\cite{csiszar2004informationtheory,liese2006divergences}, also
known as the Ali-Silvey distances~\cite{ali1966divergence}.
Given two distributions $P$ and $Q$ that possess, respectively, an absolutely
continuous density function $p$ and $q$ with respect to a base measure $\textrm{d}x$
defined on the domain $\mathcal{X}$, we define the \emph{$f$-divergence},
\begin{equation}
D_f(P \| Q) = \int_{\mathcal{X}} q(x)
	f\left(\frac{p(x)}{q(x)}\right)\,\textrm{d}x,
\label{eqn:f-divergence}
\end{equation}
where the \emph{generator function} $f: \mathbb{R}_+ \to \mathbb{R}$ is a
convex, lower-semicontinuous function satisfying $f(1)=0$.
Different choices of $f$ recover popular divergences as special cases
in~(\ref{eqn:f-divergence}).
We illustrate common choices in Table~\ref{tab:f-divergences}. See
supplementary material for more divergences and plots.

\subsection{Variational Estimation of $f$-divergences}
Nguyen et al.~\cite{nguyen2010divergenceestimation} derive a general
variational method to estimate $f$-divergences given only samples from
$P$ and $Q$.
We will extend their method from merely estimating a divergence for a 
fixed model to  estimating model parameters.
We call this new method \emph{variational divergence minimization} (VDM) and
show that the generative-adversarial training is a special case of this more
general VDM framework.

For completeness, we first provide a self-contained derivation of Nguyen et
al's
divergence estimation procedure.
%
Every convex, lower-semicontinuous function $f$ has a \emph{convex conjugate}
function $f^*$, also known as \emph{Fenchel
conjugate}~\cite{hiriarturruty2012convexanalysis}.
This function is defined as
\begin{equation}
f^*(t) = \sup_{u \in \textrm{dom}_f} \left\{ u t - f(u) \right\}.
\label{eqn:fstar}
\end{equation}
The function $f^*$ is again convex and lower-semicontinuous and the pair
$(f,f^*)$ is dual to another in the sense that $f^{**} = f$.  Therefore, we can
also represent $f$ as
$f(u) = \sup_{t \in \textrm{dom}_{f^*}} \left\{ t u - f^*(t) \right\}$.
Nguyen et al. leverage the above variational representation of $f$ in the definition
of the $f$-divergence to obtain a lower bound on the divergence,
\begin{align}
D_f(P \| Q)  &=
 \textstyle\bigintsss_{\mathcal{X}} q(x)
	\sup\limits_{t \in \textrm{dom}_{f^*}}
	\left\{t \frac{p(x)}{q(x)} - f^*(t)\right\}
	\,\textrm{d}x\notag\\
& \geq \textstyle\sup_{T \in \mathcal{T}}\left(
	\bigintsss_{\mathcal{X}} p(x) \, T(x)\,\textrm{d}x
- \textstyle\bigintsss_{\mathcal{X}}q(x) \, f^*(T(x)) \,\textrm{d}x\right)\notag\\
 &=  \sup_{T \in \mathcal{T}}
	\left(
		\mathbb{E}_{x \sim P}\left[T(x)\right]
		- \mathbb{E}_{x \sim Q}\left[f^*(T(x))\right]
	\right),
	\label{eqn:f-variational}
\end{align}
where $\mathcal{T}$ is an arbitrary class of functions $T: \mathcal{X} \to
\mathbb{R}$.
The above derivation yields a lower bound for two reasons: \emph{first},
because of Jensen's inequality when swapping the integration and supremum
operations.
\emph{Second}, the class of functions $\mathcal{T}$ may contain only a subset
of all possible functions. 

By taking the variation of the lower bound in~\eqref{eqn:f-variational} w.r.t.
$T$, we find that under mild conditions on
$f$~\cite{nguyen2010divergenceestimation}, the bound is tight for
\begin{equation}
	T^{\ast}(x) = f^{\prime}\left( \frac{p(x)}{q(x)} \right),
	\label{eqn:tstar}
\end{equation}
where $f^{\prime}$ denotes the first order derivative of $f$. This condition
can serve as a guiding principle for choosing $f$ and designing the class of 
functions $\mathcal{T}$. For example, the popular reverse Kullback-Leibler
divergence corresponds to $f(u) = -\log(u)$ resulting in
$T^{\ast}(x) = -q(x)/p(x)$, see Table~\ref{tab:f-divergences}.

We list common $f$-divergences in Table~\ref{tab:f-divergences}
and provide their Fenchel conjugates $f^*$ and the domains
$\textrm{dom}_{f^*}$ in Table~\ref{tab:f-divergence-act}.
We provide plots of the generator functions and their conjugates in the
supplementary materials.


\begin{table}[tb]
\begin{center}
\scalebox{0.75}{%
\begin{tabular}{llll}
\toprule
Name & $D_f(P\|Q)$ & Generator $f(u)$ 
& $T^{\ast}(x)$\\ \midrule
Kullback-Leibler
& $\int p(x) \log \frac{p(x)}{q(x)} \,\textrm{d}x$
& $u \log u$
& $1+\log \frac{p(x)}{q(x)}$\\
Reverse KL
& $\int q(x) \log \frac{q(x)}{p(x)}\,\textrm{d}x$
& $-\log u$
& $-\frac{q(x)}{p(x)}$\\
Pearson $\chi^2$
& $\int \frac{(q(x)-p(x))^2}{p(x)}\,\textrm{d}x$
& $(u-1)^2$
& $2(\frac{p(x)}{q(x)}-1)$\\
Squared Hellinger
& $\int\left(\sqrt{p(x)} - \sqrt{q(x)}\right)^2 \,\textrm{d}x$
& $\left(\sqrt{u}-1\right)^2$
& $(\sqrt{\frac{p(x)}{q(x)}}-1)\cdot\sqrt{\frac{q(x)}{p(x)}}$\\
Jensen-Shannon
& $\frac{1}{2} \int p(x) \log \frac{2 p(x)}{p(x)+q(x)}
	+ q(x) \log \frac{2 q(x)}{p(x) + q(x)}\,\textrm{d}x$
& $-(u+1) \log \frac{1+u}{2} + u \log u$
& $\log\frac{2p(x)}{p(x)+q(x)}$\\

\\

GAN
& $\int p(x) \log \frac{2 p(x)}{p(x)+q(x)}
	+ q(x) \log \frac{2 q(x)}{p(x) + q(x)}\,\textrm{d}x - \log(4)$
& $u \log u - (u+1) \log(u+1)$
& $\log\frac{p(x)}{p(x)+q(x)}$\\

\\
\bottomrule
\end{tabular}
}%
\end{center}
\caption{\small List of $f$-divergences $D_f(P\|Q)$ together with generator
functions.
Part of the list of divergences and their generators is based
on~\cite{nielsen2014fdivergences}.
For all divergences we have $f: \textrm{dom}_f \to \mathbb{R} \cup
\{+\infty\}$, where $f$ is convex and lower-semicontinuous.  Also we have
$f(1)=0$ which ensures that $D_f(P\|P)=0$ for any distribution $P$.
As shown by~\cite{goodfellow2014generativeadversarial}
GAN is related to the Jensen-Shannon divergence through
 $D_{\textrm{GAN}} = 2 D_{\textrm{JS}} - \log(4)$.
}
\label{tab:f-divergences}
\end{table}

\subsection{Variational Divergence Minimization (VDM)}
We now use the variational lower bound~(\ref{eqn:f-variational}) on the
$f$-divergence $D_f(P\|Q)$ in order to estimate a generative model $Q$ given
a true distribution $P$.

To this end, we follow the generative-adversarial
approach~\cite{goodfellow2014generativeadversarial} and use two neural
networks, $Q$ and $T$.
$Q$ is our generative model, taking as input a random vector and outputting a
sample of interest.  We parametrize $Q$ through a vector $\theta$ and write
$Q_{\theta}$.
$T$ is our variational function, taking as input a sample and returning a
scalar.  We parametrize $T$ using a vector $\omega$ and write $T_{\omega}$.

We can learn a generative model $Q_{\theta}$ by finding a saddle-point of the
following $f$-GAN objective function, where we minimize with respect to
$\theta$ and maximize with respect to $\omega$,
\begin{equation}
F(\theta,\omega) =
	\mathbb{E}_{x \sim P}\left[T_{\omega}(x)\right]
	- \mathbb{E}_{x \sim Q_{\theta}}\left[f^*(T_{\omega}(x))\right].
\label{eqn:F}
\end{equation}
To optimize~(\ref{eqn:F}) on a given finite training data set, we approximate
the expectations using minibatch samples.
To approximate $\mathbb{E}_{x \sim P}[\cdot]$ we sample $B$ instances without
replacement from the training set.
To approximate $\mathbb{E}_{x \sim Q_{\theta}}[\cdot]$ we sample $B$
instances from the current generative model $Q_{\theta}$.

\subsection{Representation for the Variational Function}
To apply the variational objective~(\ref{eqn:F}) for different
$f$-divergences, we need to respect the domain $\textrm{dom}_{f^*}$ of the conjugate
functions $f^*$.
To this end, we assume that variational function $T_{\omega}$ is represented
in the form $T_{\omega}(x) = g_f(\varf(x))$ and rewrite the
saddle objective \eqref{eqn:F} as follows:
\begin{equation}
 \label{eqn:Fv}
F(\theta,\omega)=
	\mathbb{E}_{x \sim P}\left[g_f(\varf(x))\right]
	 + \mathbb{E}_{x \sim Q_{\theta}}\left[-f^*(g_f(\varf(x)))\right],
\end{equation}
where $\varf: \mathcal{X} \to \mathbb{R}$ without any range constraints
on the output, and $g_f: \mathbb{R} \to \textrm{dom}_{f^*}$ is an \emph{output
activation function} specific to the $f$-divergence used.
In Table~\ref{tab:f-divergence-act} we propose suitable output activation
functions for the various conjugate functions $f^*$ and their
domains.\footnote{Note that for numerical implementation we recommend directly
implementing the scalar function $f^*(g_f(\cdot))$ robustly instead of
evaluating the two functions in sequence; see Figure
\ref{fig:objectives}.} Although the choice of $g_f$ is somewhat
arbitrary, we choose all of them to be monotone increasing functions so
that a large output $V_\omega(x)$ corresponds to the belief of the
variational function that the sample $x$ comes from the data
distribution $P$ as in the GAN case; see Figure~\ref{fig:objectives}.
It is also instructive to look at the second term $-f^\ast(g_f(v))$ in
the saddle objective \eqref{eqn:Fv}. This term is typically (except for
the Pearson $\chi^2$ divergence) a decreasing
function of the output $V_{\omega}(x)$ favoring variational functions
that output negative numbers for samples from the generator.

\begin{table}[tb]
\begin{center}
\scalebox{0.85}{%
\begin{tabular}{lllll}
\toprule
Name & Output activation $g_f$ & $\textrm{dom}_{f^*}$ & Conjugate $f^*(t)$ & $f'(1)$
\\ \midrule
Kullback-Leibler (KL)
& $v$
& $\mathbb{R}$
& $\exp(t-1)$
& $1$
\\
Reverse KL
& $-\exp(-v)$
& $\mathbb{R}_-$
& $-1-\log(-t)$
& $-1$
\\
Pearson $\chi^2$
& $v$
& $\mathbb{R}$
& $\frac{1}{4} t^2 + t$
& $0$
\\
Squared Hellinger
& $1 - \exp(-v)$
& $t < 1$
& $\frac{t}{1-t}$
& $0$
\\
Jensen-Shannon
& $\log(2) - \log(1 + \exp(-v))$
& $t < \log(2)$
& $- \log(2-\exp(t))$
& $0$
\\

GAN
& $-\log(1 + \exp(-v))$
& $\mathbb{R}_-$
& $- \log(1-\exp(t))$
& $-\log(2)$
\\
\bottomrule
\end{tabular}
}%
\end{center}
\caption{\small Recommended final layer activation functions and critical variational
function level defined by $f'(1)$.
%
%
%
The critical value $f'(1)$ can be interpreted as a classification threshold
applied to $T(x)$ to distinguish between true and generated samples.
}
\label{tab:f-divergence-act}
\end{table}

We can see the GAN objective,
\begin{align}
\label{eqn:gan}
F(\theta,\omega) =
	\mathbb{E}_{x \sim P}\left[\log D_\omega(x)\right]
	+ \mathbb{E}_{x \sim Q_{\theta}}\left[\log(1-D_{\omega}(x))\right],
\end{align}
as a special instance of \eqref{eqn:Fv} by identifying
each terms in the expectations of \eqref{eqn:Fv} and \eqref{eqn:gan}.
In particular, choosing the last nonlinearity in the discriminator
as the sigmoid
$D_{\omega}(x)=1/(1+e^{-\varf(x)})$, corresponds to output
activation function is $g_f(v)=-\log(1+e^{-v})$; see Table~\ref{tab:f-divergence-act}.

\begin{figure}[tb]
 \begin{center}
  \includegraphics[width=.8\textwidth]{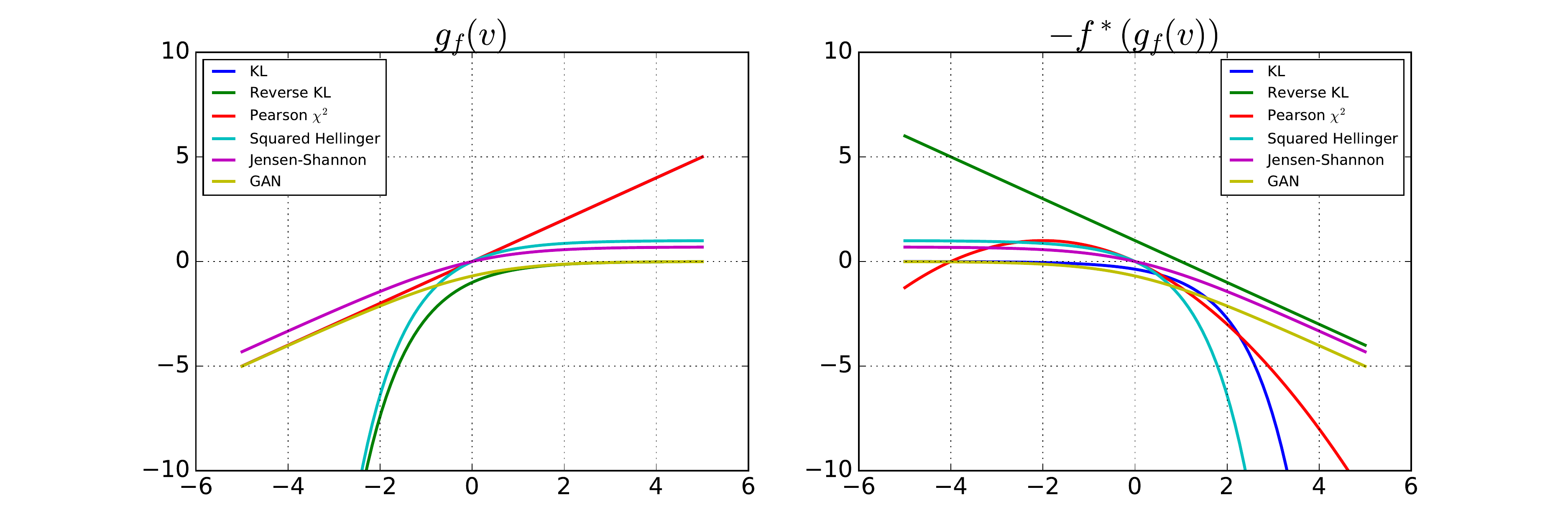}
 \caption{\small The two terms in the saddle objective \eqref{eqn:Fv}
  are plotted as a function of the variational function $\varf(x)$.}
  \label{fig:objectives}
 \end{center}
\end{figure}

\subsection{Example: Univariate  Mixture of Gaussians}
\label{sec:gmm}


\begin{table}[t]
    \begin{center}
        {\tiny
            \begin{tabular}{lccccc}
                \toprule
                                & KL        & KL-rev     & JS    & Jeffrey  & Pearson
                \\
                \midrule                
                $D_{f} (P \rvert\lvert Q_{{\theta}^{\ast}})$    &  0.2831    & 0.2480    & 0.1280    & 0.5705    & 0.6457
                \\
                $F(\hat{\omega}, \hat{\theta})$    & 0.2801    & 0.2415    & 0.1226    & 0.5151    & 0.6379
                \\   
                \midrule  
                $\mu^{\ast}$       & 1.0100    & 1.5782    & 1.3070    & 1.3218    & 0.5737
                \\
                $\hat{\mu}$     & 1.0335    & 1.5624    & 1.2854    & 1.2295   & 0.6157  
                \\    
                \midrule
                $\sigma^{\ast}$   & 1.8308    & 1.6319    & 1.7542    & 1.7034    & 1.9274
                \\
                $\hat{\sigma}$  & 1.8236    & 1.6403    & 1.7659    & 1.8087    & 1.9031
                \\
                \bottomrule
            \end{tabular}
            }
            \hspace{0.02\textwidth}
        {\tiny
            \begin{tabular}{lccccc}
                \toprule
                 train $\backslash$ test  & KL        & KL-rev     & JS    & Jeffrey  & Pearson
                \\
                \midrule                
                KL          & {\bf 0.2808}    & 0.3423    & 0.1314    & 0.5447    & 0.7345
                \\
                KL-rev      & 0.3518    & {\bf 0.2414}    & 0.1228    & 0.5794    & 1.3974
                \\
                JS          & 0.2871    & 0.2760    & {\bf 0.1210}    & 0.5260    & 0.92160
                \\
                Jeffrey    &  0.2869    & 0.2975    & 0.1247    & {\bf 0.5236}    & 0.8849
                \\
                Pearson     &  0.2970    & 0.5466    & 0.1665    & 0.7085    & {\bf 0.648}
                \\
                \bottomrule
            \end{tabular}
        }
    \end{center}
\caption{\small Gaussian approximation of a mixture of Gaussians.
Left: optimal objectives, and the learned mean and the standard deviation:
$\hat{\theta} = (\hat{\mu}, \hat{\sigma})$ (learned) and $\theta^{\ast} = (\mu^{\ast}, \sigma^{\ast})$ (best fit).
Right: objective values to the true distribution for each trained model.
For each divergence, the lowest objective function value is achieved by the model that was
trained for this divergence.}
    \label{TableGMM}
\end{table}

To demonstrate the properties of the different $f$-divergences and
to validate the variational divergence estimation framework we perform an
experiment similar to the one of~\cite{minka2005divergence}.

\textbf{Setup.}
We approximate a mixture of Gaussians by learning a Gaussian distribution.
We represent our model $Q_{\theta}$ using a linear
function which receives a random $z \sim \mathcal{N}(0,1)$ and outputs
$G_{\theta}(z) = \mu + \sigma z$,
where $\theta = (\mu, \sigma)$ are the two scalar parameters to be learned.
For the variational function $T_{\omega}$ we use a neural network with two
hidden layers having $64$ units each and $\textrm{tanh}$ activations.
%
We optimise the objective $F(\omega, \theta)$ by using the single-step
gradient method presented in Section~\ref{SecAlgorithms}.
In each step we sample batches of size $1024$ each for both $p(x)$ and $p(z)$
and we use a step-size of $\eta = 0.01$ for updating both $\omega$ and $\theta$.
We compare the results to the best fit provided by the exact optimization
of $D_f(P \| Q_{\theta})$ w.r.t. $\theta$, which is feasible in this case by
solving the required integrals in~(\ref{eqn:f-divergence}) numerically. 
We use $(\hat{\omega}, \hat{\theta})$ (learned) and $\theta^{\ast}$ (best fit) to distinguish the parameters sets used in these
two approaches.  

\textbf{Results.}
The left side of Table~\ref{TableGMM} shows the optimal divergence and
objective values $D_{f} (P \rvert\lvert Q_{{\theta}^{\ast}})$ and
$F(\hat{\omega}, \hat{\theta})$  as well as the resulting means and standard
deviations.
Note that the results are in line with the lower bound property, that is, we
have $D_{f} (P \rvert\lvert Q_{{\theta}^{\ast}}) \geq F(\hat{\omega}, \hat{\theta})$.
There is a good correspondence between the gap in objectives and the difference
between the fitted means and standard deviations. 
The right side of Table~\ref{TableGMM} shows the results of the following
experiment: (1) we train $T_{\omega}$ and $Q_{\theta}$ using a particular
divergence, then (2) we estimate the divergence and re-train $T_{\omega}$ while
keeping $Q_{\theta}$ fixed.
As expected, $Q_{\theta}$ performs best on the divergence it was trained with.
Further details showing detailed plots of the fitted Gaussians and the optimal
variational functions are presented in the supplementary materials.

In summary, the above results demonstrate that when the generative model is
misspecified and does not contain the true distribution, the divergence
function used for estimation has a strong influence on which model is learned.

\section{Algorithms for Variational Divergence Minimization (VDM)}\label{SecAlgorithms}
We now discuss numerical methods to find saddle points of the
objective \eqref{eqn:F}.
To this end, we distinguish two methods; first, the alternating method
originally proposed by Goodfellow et
al.~\cite{goodfellow2014generativeadversarial}, and second, a more direct
single-step optimization procedure.

In our variational framework,
the alternating gradient method can be described as a double-loop method;
the internal loop tightens the lower
bound on the divergence, whereas the outer loop improves the generator model.
While the motivation for this method is plausible, in practice the choice
taking a single step in the inner loop is popular.  Goodfellow et
al.~\cite{goodfellow2014generativeadversarial} provide a local convergence
guarantee.

\subsection{Single-Step Gradient Method}
\label{sec:SSG}

Motivated by  the success of the alternating gradient method with a single inner step,
we  propose a simpler algorithm shown in
Algorithm~\ref{alg:method2}. The algorithm differs from the original
one in that there is no inner loop and the gradients with respect to $\omega$ and
$\theta$ are computed in a single back-propagation.

\begin{algorithm}
\caption{\small Single-Step Gradient Method}
	\label{alg:method2}
\scalebox{0.95}{%
{\small
\begin{minipage}[h!]{0.999\linewidth}
\begin{algorithmic}[1]
\Function{SingleStepGradientIteration}{$P,\theta^t,\omega^t,B,\eta$}
\State Sample $X_P=\{x_1, \ldots, x_B\}$ and $X_Q=\{x'_1,
 \ldots, x'_B\}$, from $P$ and $Q_{\theta^t}$, respectively.
\State Update: $\omega^{t+1}= \omega^{t} + \eta \, \nabla_{\omega} F(\theta^t,\omega^t)$.
\State Update: $\theta^{t+1}=\theta^{t} - \eta \, \nabla_{\theta} F(\theta^t,\omega^t)$.
\label{alg:generator-update}
\EndFunction
\end{algorithmic}
\end{minipage}%
}
}%
\end{algorithm}

\paragraph{Analysis.}
Here we show that Algorithm~\ref{alg:method2} geometrically
converges to a saddle point $(\theta^{\ast}, \omega^{\ast})$ if there is a
neighborhood around the saddle point in which $F$ is strongly convex in
$\theta$ and strongly concave in~$\omega$. These conditions are similar to the
assumptions made in~\cite{goodfellow2014generativeadversarial} and can be
formalized as follows:
  \begin{align}
   \label{eq:saddle-conditions}
  \nabla_{\theta} F(\theta^{\ast},\omega^{\ast}) = 0,\quad
  \nabla_{\omega} F(\theta^{\ast}, \omega^{\ast}) = 0,\quad
\nabla_\theta^2 F(\theta,\omega)\succeq \delta
  I,\quad \nabla_\omega^2 F(\theta,\omega)\preceq -\delta I.
  \end{align}
These assumptions are necessary except for the ``strong'' part in order to
define the type of saddle points that are
valid solutions of our variational framework. Note that although there could be many saddle points
that arise from the structure of deep networks
\cite{DauPasGulChoGanBen14}, they do not qualify as the solution of our
variational framework under these assumptions.

For convenience, let's define $\pi^{t}=(
 \theta^{t}, \omega^{t})$.
Now the convergence of Algorithm \ref{alg:method2} can be stated as
follows (the proof is given in the supplementary material):
\begin{theorem}
\label{thm:conv}
Suppose that there is a saddle point $\pi^{\ast}=(\theta^{\ast},\omega^{\ast})$
with a neighborhood that satisfies conditions
 \eqref{eq:saddle-conditions}.
Moreover, we define $J(\pi) = \frac{1}{2}\|\nabla F(\pi)\|_2^2$ and assume 
that in the above neighborhood,
 $F$ is sufficiently smooth so that there is a constant $L>0$ such that
 $\|\nabla J(\pi')-\nabla J(\pi)\|_2\leq L\|\pi'-\pi\|_2$
 for any $\pi,\pi'$ in the neighborhood of $\pi^{\ast}$.
 Then using the step-size $\eta=\delta/L$ in Algorithm \ref{alg:method2}, we have
\begin{align*}
 J(\pi^{t}) \leq \left(1-\frac{\delta^2}{2L}\right)^{t} J(\pi^{0})
\end{align*}
 That is, the squared norm
 of the gradient $\nabla F(\pi)$ decreases geometrically.
\end{theorem}

\subsection{Practical Considerations}
Here we discuss principled extensions of the heuristic proposed in 
\cite{goodfellow2014generativeadversarial} and real/fake statistics
discussed by Larsen and
S{\o}nderby\footnote{http://torch.ch/blog/2015/11/13/gan.html}. Furthermore
we discuss practical advice that 
slightly deviate from the principled viewpoint.

Goodfellow et al. \cite{goodfellow2014generativeadversarial} noticed
that training GAN can be significantly sped up by maximizing
$\mathbb{E}_{x\sim Q_\theta}\left[\log D_{\omega}(x)\right]$ instead of
minimizing $\mathbb{E}_{x\sim
Q_\theta}\left[\log\left(1-D_{\omega}(x)\right)\right]$ for updating the
generator.
In the more general $f$-GAN Algorithm~\eqref{alg:method2} this means that we
replace line~\ref{alg:generator-update} with the update
\begin{equation}
\theta^{t+1} = \theta^t + \eta \,
	\nabla_{\theta} \mathbb{E}_{x \sim Q_{\theta^t}}[g_f(V_{\omega^t}(x))],
\end{equation}
thereby maximizing the generator output.
This is not only intuitively correct but we can show that the stationary point
is preserved by this change using the same argument as
in~\cite{goodfellow2014generativeadversarial}; we found this useful also for
other divergences.

Larsen and S{\o}nderby recommended monitoring \emph{real} and \emph{fake} statistics,
which are defined as the true positive and true negative rates of the
variational function viewing it as a binary classifier. Since our output
activation $g_f$ are all monotone, we can derive similar statistics for
any $f$-divergence by only changing the decision threshold. Due to the
link between the density ratio and the variational function
\eqref{eqn:tstar}, the threshold lies at $f'(1)$ (see Table
\ref{tab:f-divergence-act}). That is, we can interpret the output of
the variational function as classifying the input $x$ as a true
sample if the variational function
$T_\omega(x)$ is larger than $f'(1)$, and classifying it as a
sample from the generator otherwise.


We found Adam \cite{kingma2014adam} and gradient clipping to be useful
especially in the large scale experiment on the LSUN dataset.


\section{Experiments}
\label{sec:experiments}

We now train generative neural samplers based on VDM on the MNIST and
LSUN datasets.

\paragraph{MNIST Digits.}
We use the MNIST training data set (60,000 samples, 28-by-28 pixel images) to
train the generator and variational function model proposed
in~\cite{goodfellow2014generativeadversarial} for various $f$-divergences.
With $z \sim \textrm{Uniform}_{100}(-1,1)$ as input, the generator model
has two linear layers each followed by batch normalization and ReLU
activation and a final linear layer followed by the sigmoid function.
The variational function $\varf(x)$ has three linear layers with
exponential linear unit~\cite{clevert2015elu} in between. The final
activation is specific to each divergence and listed in Table \ref{tab:f-divergence-act}.
As in~\cite{radford2015dcgan} we use Adam with a learning rate of
$\alpha=0.0002$ and update weight $\beta=0.5$.  We use a batchsize of 4096,
sampled from the training set without replacement, and train each model for
one hour.
We also compare against variational autoencoders~\cite{kingma2013vae}
with 20 latent dimensions.

{\em Results and Discussion.}
We evaluate the performance using the kernel density estimation (Parzen
window) approach used in~\cite{goodfellow2014generativeadversarial}.
To this end, we sample 16k images from the model and estimate a Parzen window
estimator using an isotropic Gaussian kernel bandwidth using three fold cross
validation.  The final density model is used to evaluate the average
log-likelihood on the MNIST test set (10k samples).
We show the results in Table~\ref{tab:mnist-kde}, and some samples from our
models in Figure~\ref{fig:mnist-samples}.

The use of the KDE approach to log-likelihood estimation has known
deficiencies~\cite{theis2015generativemodelevaluation}.
In particular, for the dimensionality used in MNIST ($d=784$) the number of
model samples required to obtain accurate log-likelihood estimates is
infeasibly large.
We found a large variability (up to 50 nats) between multiple repetitions.
As such the results are not entirely conclusive.
We also trained the same KDE estimator on the MNIST training set, achieving a
significantly higher holdout likelihood.
However, it is reassuring to see that the model trained for the
Kullback-Leibler divergence indeed achieves a high holdout likelihood compared
to the GAN model.

%
\begin{minipage}[t]{\textwidth}
\begin{minipage}[b]{0.729\textwidth}
\centering
{\small%
\begin{tabular}{lrr}
\toprule
Training divergence & KDE $\langle LL \rangle$ (nats) & $\pm$ SEM
\\ \midrule
Kullback-Leibler
& 416 & 5.62		
\\
 Reverse Kullback-Leibler
& 319 & 8.36		
\\
Pearson $\chi^2$
& 429 & 5.53		
\\
Neyman $\chi^2$
& 300 & 8.33		
\\
Squared Hellinger
& -708 & 18.1      
\\
Jeffrey
& -2101 & 29.9     
\\
Jensen-Shannon
& 367 & 8.19		
\\
GAN
& 305 & 8.97       
\\ \midrule
Variational Autoencoder~\cite{kingma2013vae}
& 445 & 5.36		
\\
KDE MNIST train (60k)
& 502 & 5.99
\\
\bottomrule
\end{tabular}
}%
\captionof{table}{\small Kernel Density Estimation evaluation on the MNIST test data set.
Each KDE model is build from 16,384 samples from the learned generative model.
We report the mean log-likelihood on the MNIST test set ($n=10,000$) and the
standard error of the mean.
The KDE MNIST result is using 60,000 MNIST training images to fit a single KDE
model.}
\label{tab:mnist-kde}
\end{minipage}%
\hfill%
\begin{minipage}[b]{0.249\textwidth}
\centering
\includegraphics[width=0.82\linewidth]{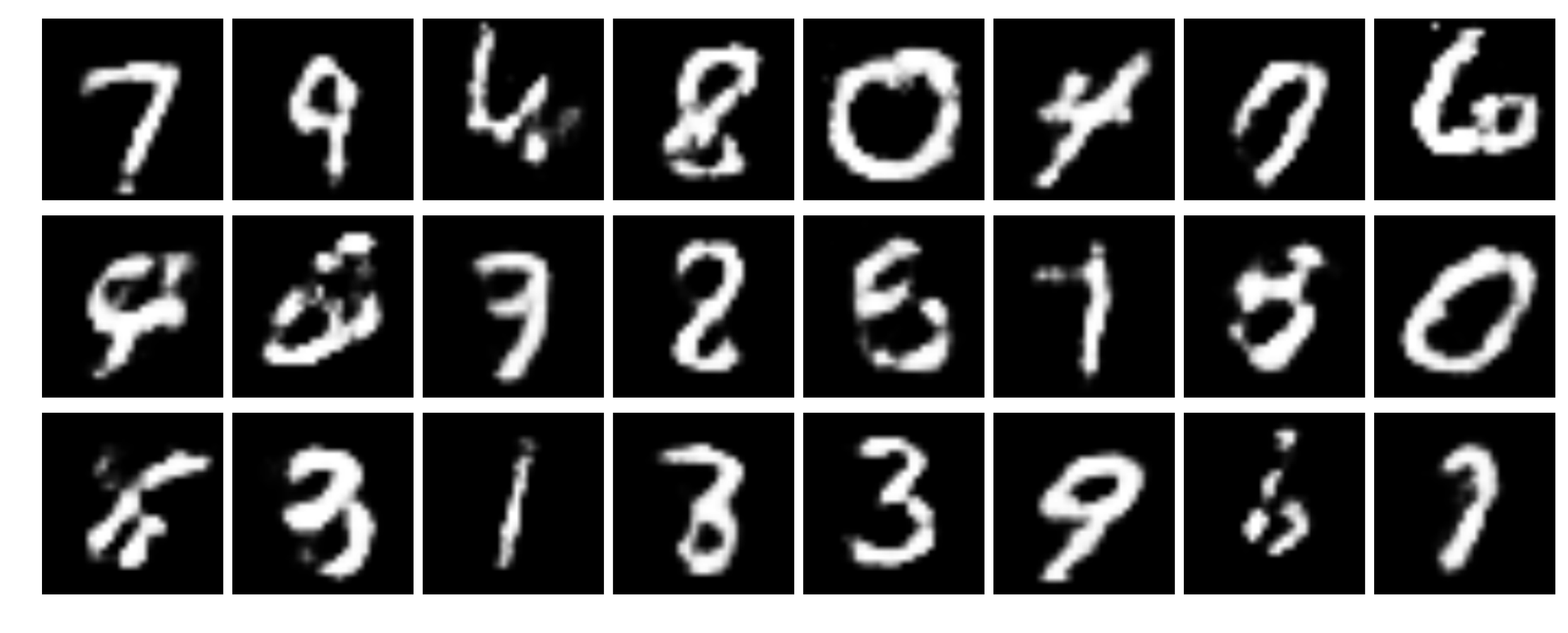}\\%
\includegraphics[width=0.82\linewidth]{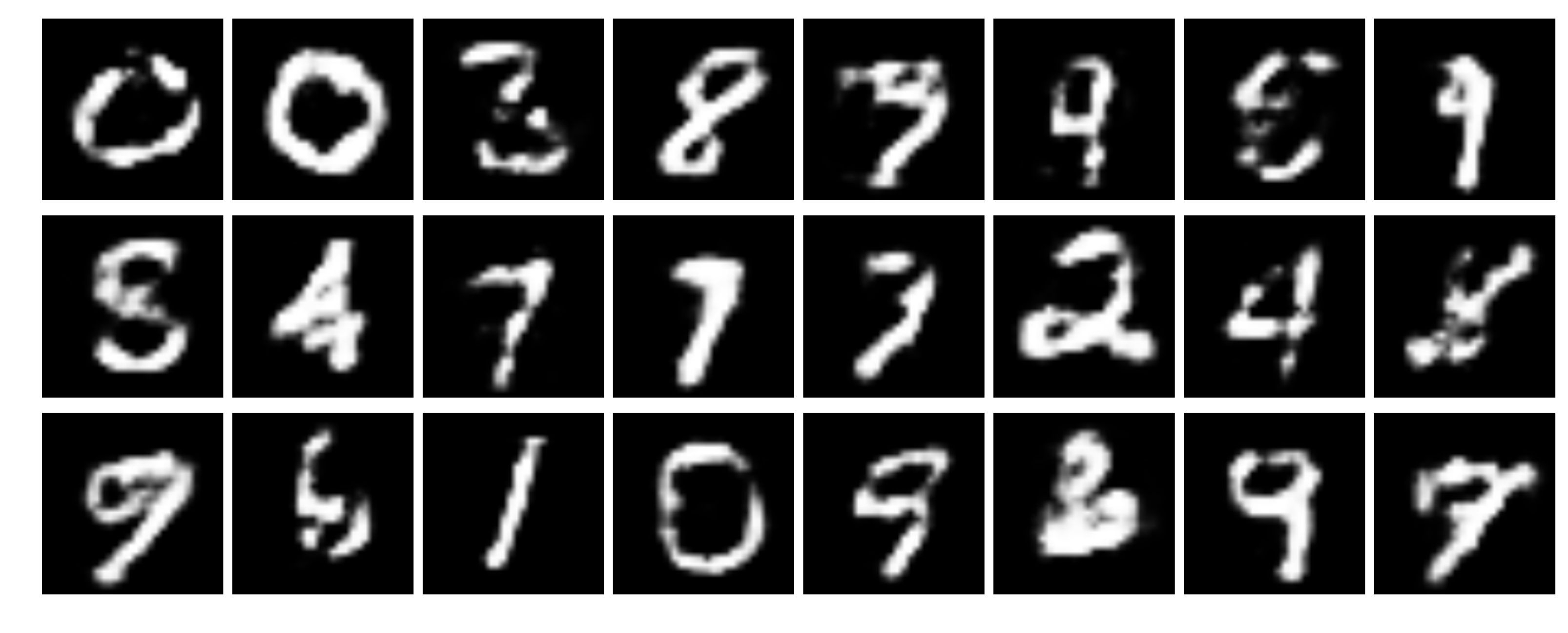}\\%
\includegraphics[width=0.82\linewidth]{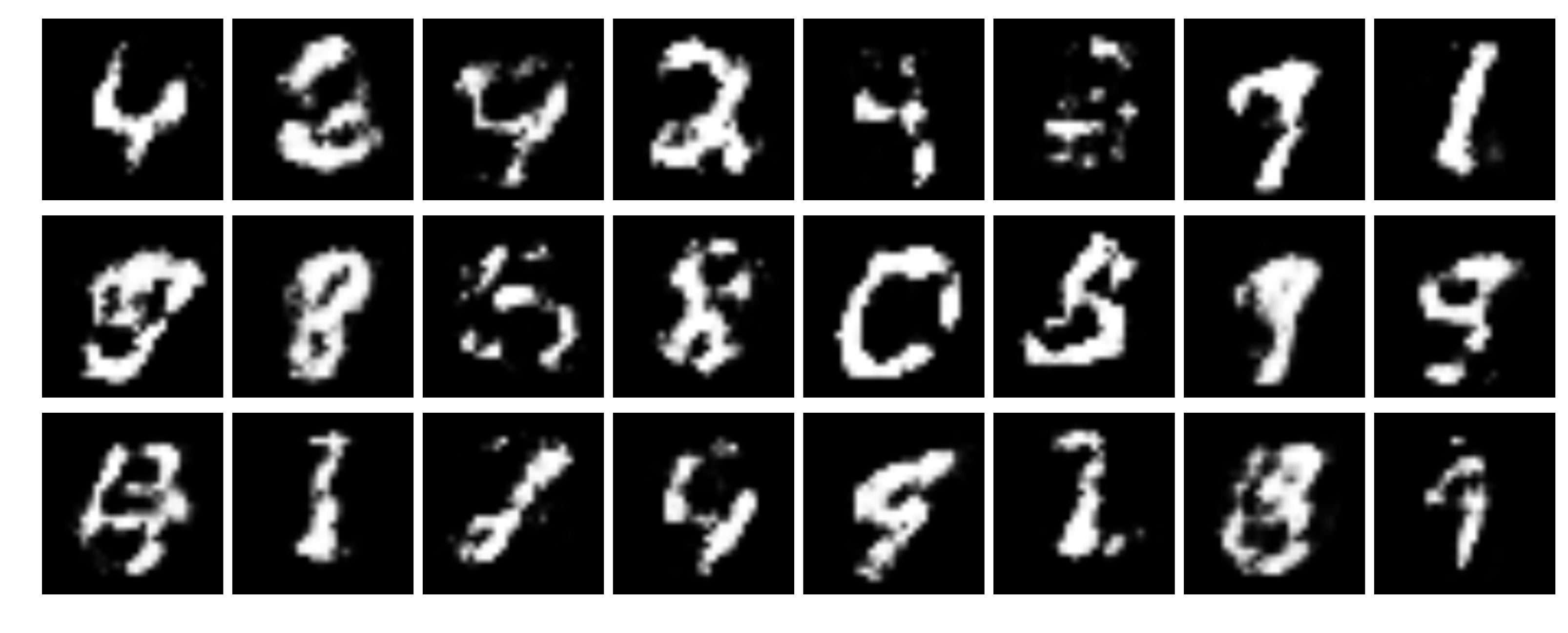}\\%
\includegraphics[width=0.82\linewidth]{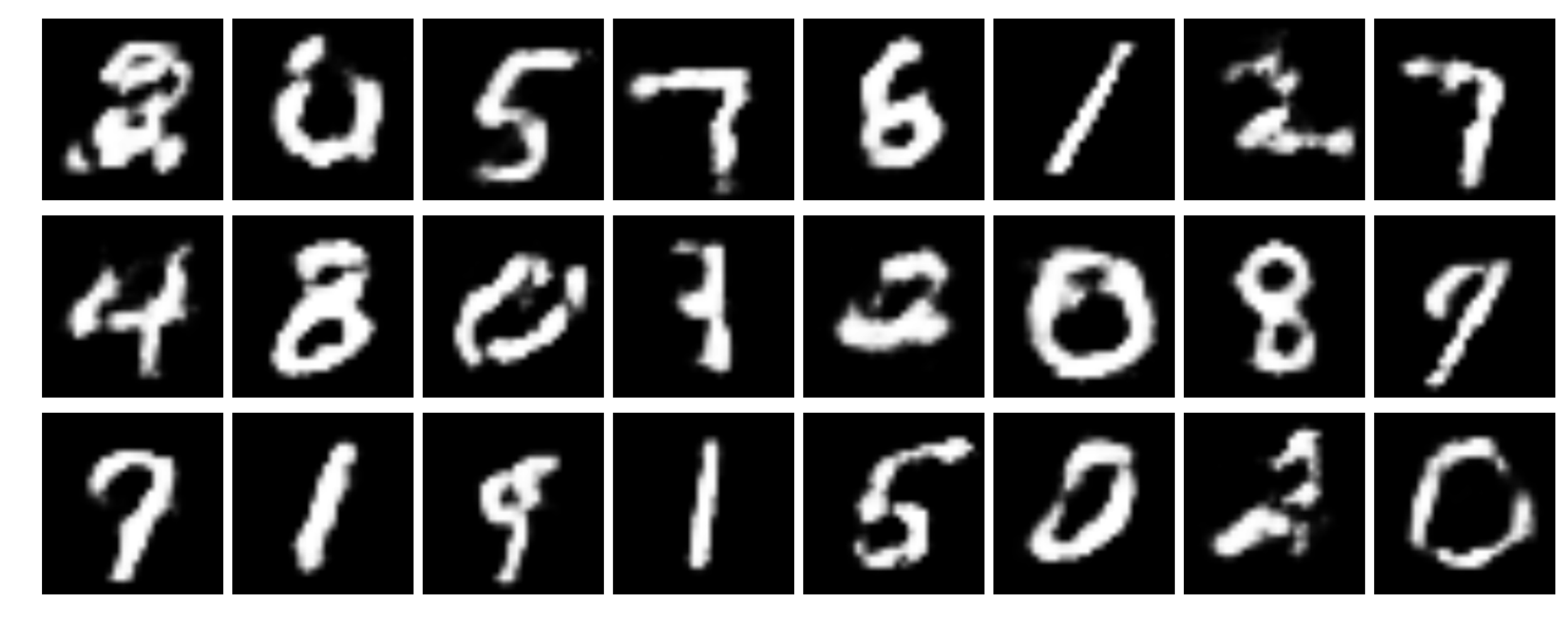}%
\captionof{figure}{\small MNIST model samples trained using
KL, reverse KL, Hellinger, Jensen from top to bottom.}
\label{fig:mnist-samples}%
\end{minipage}
\end{minipage}

\paragraph{LSUN Natural Images.}
Through the DCGAN work~\cite{radford2015dcgan} the generative-adversarial
approach has shown real promise in generating natural looking images.
Here we use the same architecture as as in~\cite{radford2015dcgan} and replace
the GAN objective with our more general $f$-GAN objective.

We use the large scale LSUN database~\cite{yu15lsun} of natural images of
different categories.
To illustrate the different behaviors of different divergences we train the
same model on the \emph{classroom} category of images, containing 168,103
images of classroom environments, rescaled and center-cropped to 96-by-96
pixels.

{\em Setup.}
We use the generator architecture and training settings proposed in
DCGAN~\cite{radford2015dcgan}. The model receives $z\in {\rm Uniform}_{d_{\rm rand}}(-1,1)$
and feeds it through one linear layer and three deconvolution layers
with batch normalization and ReLU activation in between.
The variational function is the same as the discriminator architecture
in~\cite{radford2015dcgan} and follows the structure of a convolutional neural
network with batch normalization, exponential linear
units~\cite{clevert2015elu} and one final linear layer.

{\em Results.} Figure \ref{fig:lsun} shows 16 random samples from neural
samplers trained using GAN, KL, and squared Hellinger divergences.
All three divergences produce equally realistic samples. Note that 
the difference in the learned distribution $Q_{\theta}$ arise only when
the generator model is not rich enough. 

\begin{figure}[hb!]
 \begin{center}
  \subfigure[GAN]
  {\fbox{\includegraphics[clip,width=0.32\textwidth]{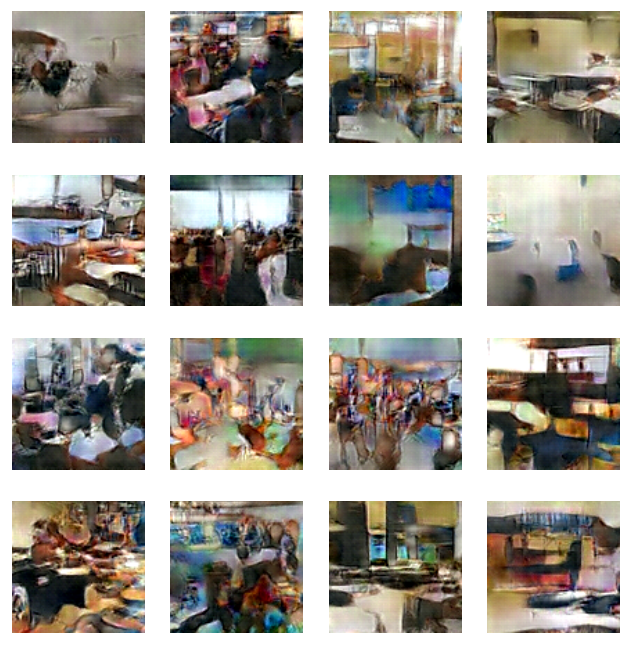}}}~\subfigure[KL]
  {\fbox{\includegraphics[clip,width=0.32\textwidth]{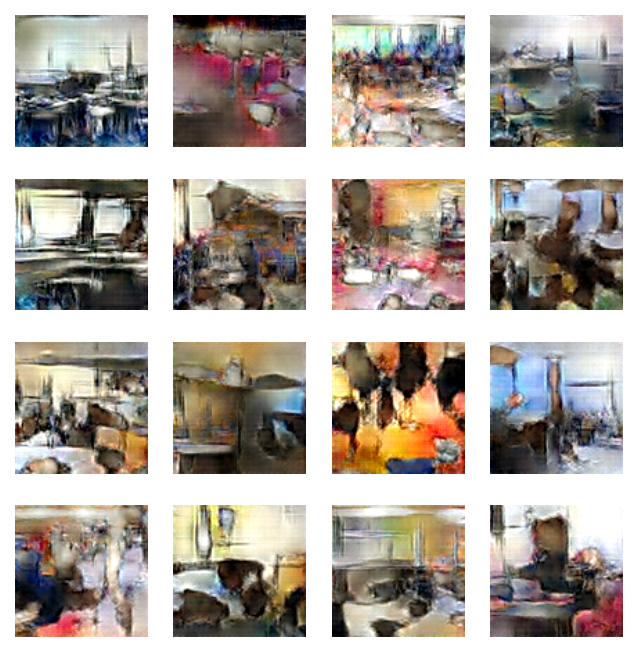}}}~\subfigure[Squared Hellinger]
  {\fbox{\includegraphics[clip,width=0.32\textwidth]{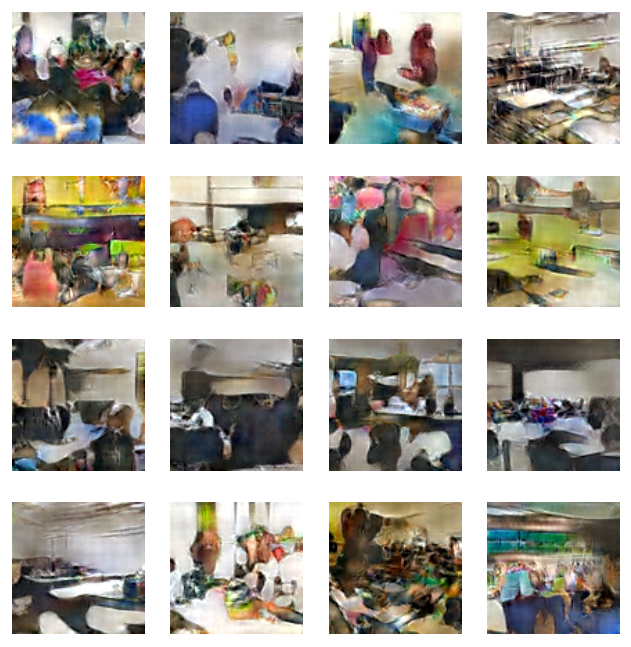}}}
  \caption{Samples from three different divergences.}
  \label{fig:lsun}
 \end{center}
\end{figure}

\section{Related Work}
We now discuss how our approach relates to existing work.
Building generative models of real world distributions is a fundamental goal
of machine learning and much related work exists.
We only discuss work that applies to neural network models.

\emph{Mixture density networks}~\cite{bishop1994mixturedensitynetworks} are
neural networks which directly regress the parameters of a finite parametric
mixture model.  When combined with a recurrent neural network this yields
impressive generative models of handwritten text~\cite{graves2013sequences}.

\emph{NADE}~\cite{larochelle2011nade} and
\emph{RNADE}~\cite{uria2013rnade} perform a
factorization of the output using a predefined and somewhat arbitrary ordering
of output dimensions.  The resulting model samples one variable at a time conditioning
on the entire history of past variables.
These models provide tractable likelihood evaluations and compelling results
but 
it is unclear how to select the factorization order in many applications .

\emph{Diffusion probabilistic models}~\cite{sohldickstein2015unsupervisedlearning}
define a target distribution as a result of a learned diffusion process which
starts at a trivial known distribution.  The learned model provides exact
samples and approximate log-likelihood evaluations.



\emph{Noise contrastive estimation} (NCE)~\cite{Gutmann10nce} is a method
that estimates the parameters of unnormalized probabilistic models by performing 
non-linear logistic regression to discriminate the data from
artificially generated noise. NCE can be viewed as a special case of GAN where the
discriminator is constrained to a specific form that depends on the model (logistic regression
classifier) and the generator (kept fixed) is providing the artificially
generated noise (see supplementary material).


The generative neural sampler models
of~\cite{mackay1995bayesianneuralnetworks} and~\cite{bishop1998gtm} did not
provide satisfactory learning methods;
\cite{mackay1995bayesianneuralnetworks} used importance sampling
and~\cite{bishop1998gtm} expectation maximization.
The main difference to GAN and to our work really is in the learning
objective, which is effective and computationally inexpensive.


\emph{Variational auto-encoders} (VAE)~\cite{kingma2013vae,rezende2014stochasticbackpropagation} are
pairs of probabilistic encoder and decoder models which map a sample to a
latent representation and back, trained using a variational Bayesian learning
objective.  The advantage of VAEs is in the encoder model which allows
efficient inference from observation to latent representation and overall they
are a compelling alternative to $f$-GANs and recent work has studied
combinations of the two approaches~\cite{makhzani2015adversarialautoencoders}


As an alternative to the GAN training objective the work~\cite{li2015mmd} and
independently~\cite{dziugaite2015mmd} considered the use of the \emph{kernel
maximum mean discrepancy}
(MMD)~\cite{gretton2007kernelmmd,gneiting2007scoringrules} as a training
objective for probabilistic models.
This objective is simpler to train compared to GAN models because there is
no explicitly represented variational function.
However, it requires the choice of a kernel
function 
and the reported results so
far seem slightly inferior compared to GAN.
MMD is a particular instance of a larger class of probability
metrics~\cite{sriperumbudur2010measuremetrics} which all take the form
$D(P,Q) = \sup_{T \in \mathcal{T}} \left|\mathbb{E}_{x \sim P}[T(x)] -
\mathbb{E}_{x \sim Q}[T(x)]\right|$, where the function class $\mathcal{T}$ is chosen
in a manner specific to the divergence.  Beyond MMD other popular metrics of
this form are the total variation metric (also an $f$-divergence), the
Wasserstein distance, and the Kolmogorov distance.


In \cite{huszar2015generativemodel} a generalisation of the GAN
objective is proposed by using an \emph{alternative Jensen-Shannon
divergence} that mimics an interpolation between the KL and the reverse
KL divergence and has Jensen-Shannon as its mid-point.
It can be shown that with $\pi$ close to $0$ and $1$ it leads to a behavior
similar the objectives resulting from the KL and reverse KL divergences
(see supplementary material).


\section{Discussion}

Generative neural samplers offer a powerful way to represent complex
distributions without limiting factorizing assumptions.
However, while the purely generative neural samplers as used in this
paper are interesting their use is limited because after training they cannot
be conditioned on observed data and thus are unable to provide inferences.

We believe that in the future the true benefits of neural samplers for
representing uncertainty will be found in discriminative models and our
presented methods extend readily to this case by providing additional inputs
to both the generator and variational function as in the conditional GAN
model~\cite{Gauthier2014}.

\emph{Acknowledgements.}  We thank Ferenc Husz\'{a}r for discussions on the
generative-adversarial approach.

{\small
\bibliographystyle{abbrvnat}
\bibliography{paper_arxiv}
}

\cleardoublepage

\begin{appendices}
\vspace{0.75cm}%
\begin{center}
{\Huge Supplementary Materials}
\end{center}
\vspace{0.75cm}%

\section{Introduction}
We provide additional material to support the content presented in the paper.
The text is structured as follows.
In Section~\ref{SecDivs} we present an extended list of f-divergences,
corresponding generator functions and their convex conjugates. In
Section~\ref{SecTheroremProof} we provide the proof of Theorem~\ref{thm:conv} from
Section~\ref{SecAlgorithms}. In Section~\ref{SecAlgorithmsRelated} we discuss the
differences between current (to our knowledge) GAN optimisation algorithms.
Section~\ref{SecGMM} provides a proof of concept of our approach by fitting
a Gaussian to a mixture of Gaussians using various divergence measures.
Finally, in Section~\ref{SecExpNN} we present the details of the network
architectures used in Section~\ref{sec:experiments} of the main text.

\section{$f$-divergences and Generator-Conjugate Pairs}\label{SecDivs}

In Table~\ref{tab:f-divergences} we show an extended list of f-divergences $D_f(P\|Q)$ together with their generators $f(u)$ and the corresponding optimal variational functions $T^{\ast}(x)$. 
For all divergences we have $f: \textrm{dom}_f \to \mathbb{R} \cup
\{+\infty\}$, where $f$ is convex and lower-semicontinuous.  Also we have
$f(1)=0$ which ensures that $D_f(P\|P)=0$ for any distribution $P$.
As shown by~\cite{goodfellow2014generativeadversarial}
GAN is related to the Jensen-Shannon divergence through
$D_{\textrm{GAN}} = 2 D_{\textrm{JS}} - \log(4)$.  The GAN generator function
$f$ does not satisfy $f(1)=0$ hence $D_{\textrm{GAN}}(P\|P) \neq 0$.

Table~\ref{tab:f-divergence-act} lists the convex conjugate functions $f^{\ast}(t)$ of the generator functions $f(u)$ in Table~\ref{tab:f-divergences}, their domains, as well as the  activation functions $g_f$ we use in the last layers of the generator networks to obtain a correct mapping of the network outputs into the domains of the conjugate functions.

The panels of Figure~\ref{fig:f} show the generator functions and the corresponding convex conjugate functions for a variety of f-divergences.

\begin{table}[h!bt]
\begin{center}
\scalebox{0.625}{%
\begin{tabular}{llll}
\toprule
Name & $D_f(P\|Q)$ & Generator $f(u)$ 
& $T^{\ast}(x)$\\ \midrule
Total variation
& $\frac{1}{2} \int |p(x)-q(x)|\,\textrm{d}x$
& $\frac{1}{2} |u-1|$
& $\frac{1}{2}{\rm sign}(\frac{p(x)}{q(x)}-1)$\\
Kullback-Leibler
& $\int p(x) \log \frac{p(x)}{q(x)} \,\textrm{d}x$
& $u \log u$
& $1+\log \frac{p(x)}{q(x)}$\\
Reverse Kullback-Leibler
& $\int q(x) \log \frac{q(x)}{p(x)}\,\textrm{d}x$
& $-\log u$
& $-\frac{q(x)}{p(x)}$\\
Pearson $\chi^2$
& $\int \frac{(q(x)-p(x))^2}{p(x)}\,\textrm{d}x$
& $(u-1)^2$
& $2(\frac{p(x)}{q(x)}-1)$\\
Neyman $\chi^2$
& $\int \frac{(p(x)-q(x))^2}{q(x)}\,\textrm{d}x$
& $\frac{(1-u)^2}{u}$
& $1 - \big[\frac{q(x)}{p(x)}\big]^{2}$ \\
Squared Hellinger
& $\int\left(\sqrt{p(x)} - \sqrt{q(x)}\right)^2 \,\textrm{d}x$
& $\left(\sqrt{u}-1\right)^2$
& $(\sqrt{\frac{p(x)}{q(x)}}-1)\cdot\sqrt{\frac{q(x)}{p(x)}}$\\
Jeffrey   
& $\int\left(p(x)-q(x)\right) \log \left(\frac{p(x)}{q(x)}\right) \,\textrm{d}x$
& $(u-1) \log u$
& $1+\log \frac{p(x)}{q(x)}-\frac{q(x)}{p(x)}$\\
Jensen-Shannon
& $\frac{1}{2} \int p(x) \log \frac{2 p(x)}{p(x)+q(x)}
  + q(x) \log \frac{2 q(x)}{p(x) + q(x)}\,\textrm{d}x$
& $-(u+1) \log \frac{1+u}{2} + u \log u$
& $\log\frac{2p(x)}{p(x)+q(x)}$\\

Jensen-Shannon-weighted
& {\small $  \int p(x) \pi \log \frac{ p(x)}{\pi p(x)+ (1-\pi)q(x)}
  + (1-\pi)q(x) \log \frac{q(x)}{\pi p(x) + (1-\pi)q(x)}\,\textrm{d}x$ }
& {\small $ \pi  u \log u - ( 1-\pi + \pi u) \log ( 1-\pi + \pi u) $ }
& {\small $\pi \log \frac{p(x)}{(1-\pi) q(x) + \pi p(x)}$ }
\\

GAN
& $\int p(x) \log \frac{2 p(x)}{p(x)+q(x)}
  + q(x) \log \frac{2 q(x)}{p(x) + q(x)}\,\textrm{d}x - \log(4)$
& $u \log u - (u+1) \log(u+1)$
& $\log\frac{p(x)}{p(x)+q(x)}$\\
$\alpha$-divergence ($\alpha \notin \{0,1\}$)
& $\frac{1}{\alpha (\alpha-1)} \int
  \left(p(x) \left[\left(\frac{q(x)}{p(x)}\right)^{\alpha}-1\right] - \alpha(q(x)-p(x))\right)
  \,\textrm{d}x$
& $\frac{1}{\alpha(\alpha-1)} \left(u^{\alpha} - 1 - \alpha(u-1)\right)$
& $\frac{1}{\alpha-1} \Big[ \big[\frac{p(x)}{q(x)}\big]^{\alpha-1} - 1\Big]$
\\
\bottomrule
\end{tabular}
}%
\end{center}
\caption{\small List of $f$-divergences $D_f(P\|Q)$ together with generator
functions and the optimal variational functions.
}
\label{tab:f-divergences}
\end{table}

\begin{table}[tb]
\begin{center}
\scalebox{0.8}{%
\begin{tabular}{lllll}
\toprule
Name & Output activation $g_f$ & $\textrm{dom}_{f^*}$ & Conjugate $f^*(t)$ & $f'(1)$
\\ \midrule
Total variation
& $\frac{1}{2} \tanh(v)$
& $-\frac{1}{2} \leq t \leq \frac{1}{2}$
& $t$
& $0$
\\
Kullback-Leibler (KL)
& $v$
& $\mathbb{R}$
& $\exp(t-1)$
& $1$
\\
Reverse KL
& $-\exp(v)$
& $\mathbb{R}_-$
& $-1-\log(-t)$
& $-1$
\\
Pearson $\chi^2$
& $v$
& $\mathbb{R}$
& $\frac{1}{4} t^2 + t$
& $0$
\\
Neyman $\chi^2$
& $1 - \exp(v)$
& $t < 1$
& $2 - 2\sqrt{1-t}$
& $0$
\\
Squared Hellinger
& $1 - \exp(v)$
& $t < 1$
& $\frac{t}{1-t}$
& $0$
\\
Jeffrey
& $v$
& $\mathbb{R}$
& $W(e^{1-t}) + \frac{1}{W(e^{1-t})} + t - 2$
& $0$
\\
Jensen-Shannon
& $\log(2) - \log(1 + \exp(-v))$
& $t < \log(2)$
& $- \log(2-\exp(t))$
& $0$
\\
Jensen-Shannon-weighted
& $ - \pi \log \pi - \log(1 + \exp(-v))$
& $t < - \pi \log \pi $
& $ (1-\pi) \log \frac{1-\pi}{ 1- \pi e^{t/\pi}}$
& $0$
\\

GAN
& $-\log(1 + \exp(-v))$
& $\mathbb{R}_-$
& $- \log(1-\exp(t))$
& $-\log(2)$
\\
$\alpha$-div. ($\alpha < 1$, $\alpha \neq 0$)
& $\frac{1}{1-\alpha} - \log(1+\exp(-v))$
& $t < \frac{1}{1-\alpha}$
& $\frac{1}{\alpha}(t(\alpha-1)+1)^{\frac{\alpha}{\alpha-1}} - \frac{1}{\alpha}$
& $0$
\\
$\alpha$-div. ($\alpha > 1$)
& $v$
& $\mathbb{R}$
& $\frac{1}{\alpha}(t(\alpha-1)+1)^{\frac{\alpha}{\alpha-1}} - \frac{1}{\alpha}$
& $0$
\\
\bottomrule
\end{tabular}
}%
\end{center}
\caption{\small Recommended final layer activation functions and critical variational
function level defined by $f'(1)$.
The objective function for training a generative neural sampler $Q_{\theta}$
given a true distribution $P$ and an auxiliary variational function $T$ is
$\min_{\theta} \max_{\omega} F(\theta,\omega) = \mathbb{E}_{x \sim
P}[T_{\omega}(x)] - \mathbb{E}_{x \sim Q_{\theta}}[f^*(T_{\omega}(x))]$.
%
For any sample $x$ the variational function produces a scalar $v(x) \in \mathbb{R}$.
The output activation provides a differentiable map $g_f: \mathbb{R} \to
\textrm{dom}_{f^*}$, defining $T(x) = g_f(v(x))$.
The critical value $f'(1)$ can be interpreted as a classification threshold
applied to $T(x)$ to distinguish between true and generated samples.
$W$ is the Lambert-$W$ product log function.
}
\label{tab:f-divergence-act}
\end{table}

\begin{figure}[t]\centering%
    \includegraphics[width=0.475\linewidth]{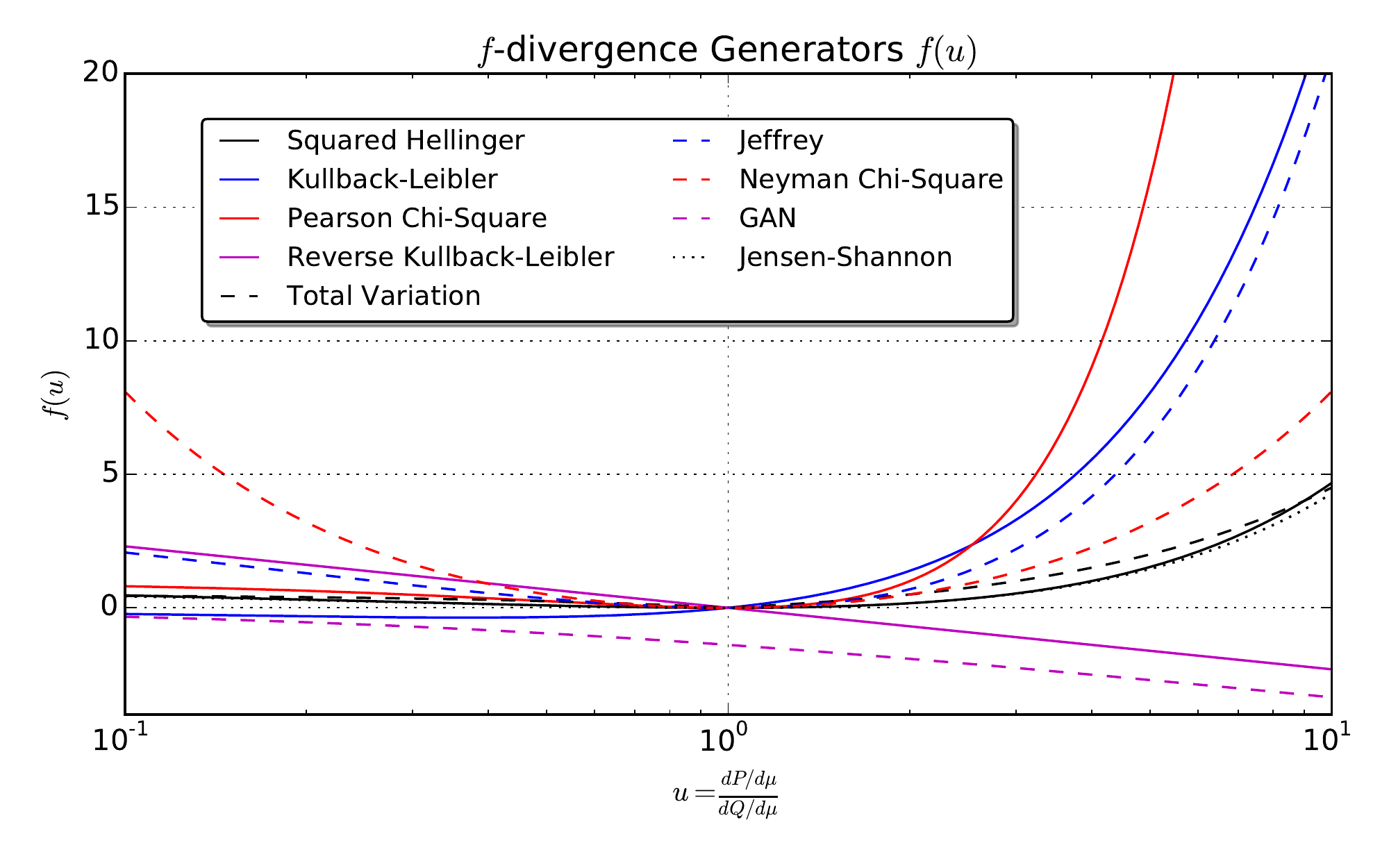}%
  \hfill%
    \includegraphics[width=0.475\linewidth]{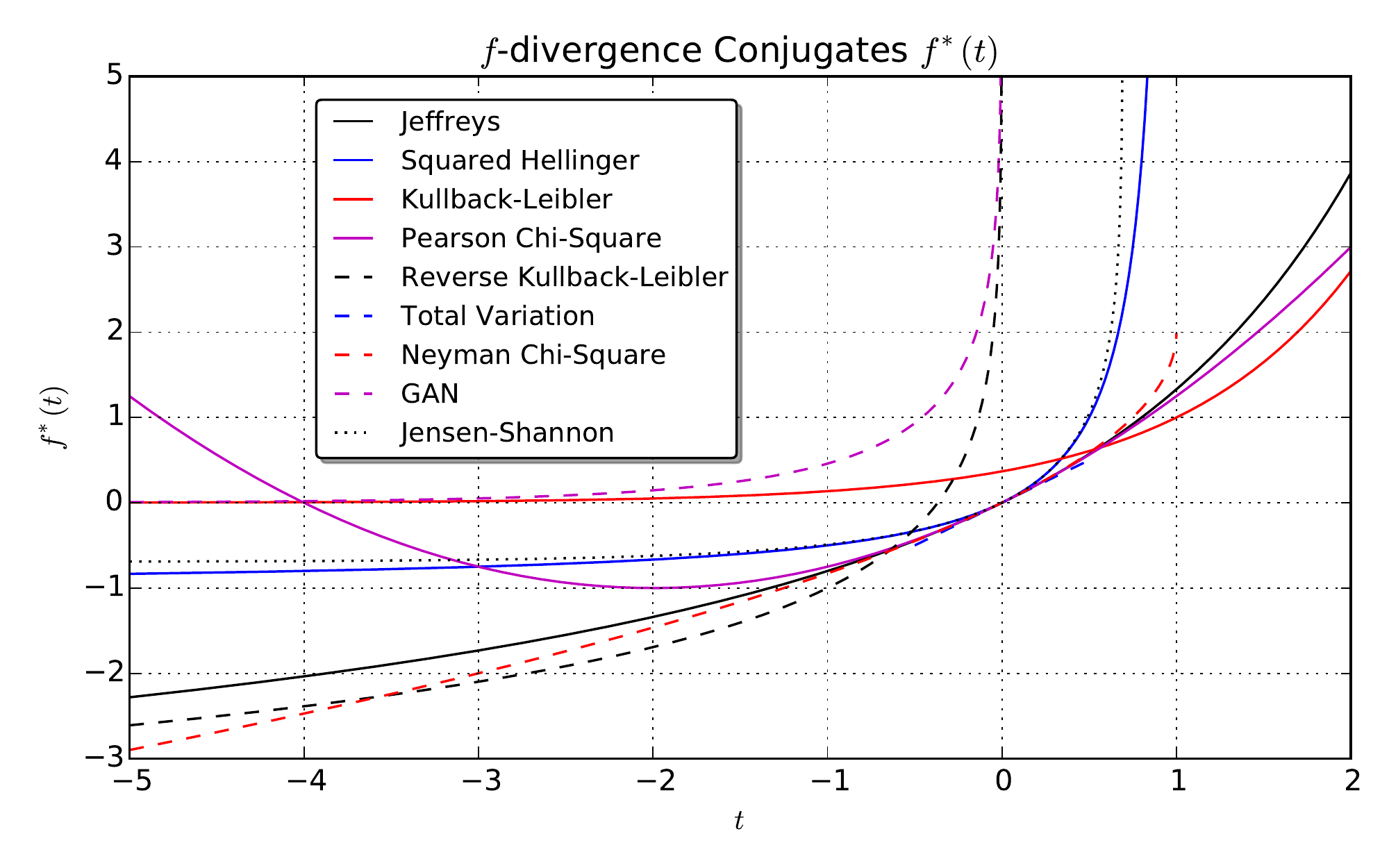}%
  \caption{Generator-conjugate $(f,f^*)$ pairs in the variational framework of
  Nguyen et al.~\cite{nguyen2010divergenceestimation}. Left: generator functions $f$ used in the $f$-divergence
  $D_f(P\|Q) = \int_{\mathcal{X}} q(x) f\left(\frac{p(x)}{q(x)}\right)\,\textrm{d}x$. Right: conjugate functions $f^*$ in the variational
  divergence lower bound $D_f(P\|Q) \geq \sup_{T \in \mathcal{T}} \int_{\mathcal{X}}
  p(x) \, T(x) - q(x) f^*(T(x)) \,\textrm{d}x$.}
\label{fig:f}%
\end{figure}

\section{Proof of Theorem~1} \label{SecTheroremProof}
In this section we present the proof of Theorem~\ref{thm:conv} from
Section~\ref{SecAlgorithms} of the main text. For completeness, we reiterate the conditions and the theorem.

We assume that  $F$ is strongly convex in
$\theta$ and strongly concave in $\omega$ such that
  \begin{align}
   \label{eq:saddle-condition1}
  \nabla_{\theta} F(\theta^{\ast},\omega^{\ast}) = 0,\quad
  \nabla_{\omega} F(\theta^{\ast}, \omega^{\ast}) = 0,  \\
   \label{eq:saddle-condition2}
\nabla_\theta^2 F(\theta,\omega)\succeq \delta
  I,\quad \nabla_\omega^2 F(\theta,\omega)\preceq -\delta I.
  \end{align}
These assumptions are necessary except for the ``strong'' part in order to define the type of saddle points that are
valid solutions of our variational framework.

We define  $\pi^{t}=(
 \theta^{t}, \omega^{t})$ and use the notation
\begin{align*}
 \nabla F(\pi) =
 \begin{pmatrix}
  \nabla_{\theta} F(\theta,\omega) \\ \nabla_{\omega} F(\theta,\omega)
 \end{pmatrix},\quad
\tilde{\nabla} F(\pi) = 
 \begin{pmatrix}
  -\nabla_{\theta} F(\theta,\omega) \\ \nabla_{\omega} F(\theta,\omega)
 \end{pmatrix}.
\end{align*}
With this notation, Algorithm~\ref{alg:method2} in the main text can be written as
\begin{align*}
 \pi^{t+1} &= \pi^{t} + \eta \tilde{\nabla} F(\pi^{t}).
\end{align*}

Given the above assumptions and notation, in Section~\ref{SecAlgorithms} of the main text we formulate the following theorem.
\begin{theorem}
\label{thm:conv}
Suppose that there is a saddle point $\pi^{\ast}=(\theta^{\ast},\omega^{\ast})$
with a neighborhood that satisfies conditions
 \eqref{eq:saddle-condition1} and \eqref{eq:saddle-condition2}.
Moreover we define $J(\pi) = \frac{1}{2}\|\nabla F(\pi)\|_2^2$ and assume 
that in the above neighborhood,
 $F$ is sufficiently smooth so that there is a constant $L>0$ and
 \begin{align}
  \label{eq:smoothness}
J(\pi') \leq J(\pi) + \inner{\nabla J(\pi)}{\pi'-\pi} + \frac{L}{2}\|\pi'-\pi\|_2^2
 \end{align}
 for any $\pi,\pi'$ in the neighborhood of $\pi^{\ast}$.
 Then using the step-size $\eta=\delta/L$ in Algorithm~\ref{alg:method2}, we have
\begin{align*}
 J(\pi^{t}) \leq \left(1-\frac{\delta^2}{2L}\right)^{t} J(\pi^{0})
\end{align*}
where $L$ is the smoothness parameter of $J$. That is, the squared norm
 of the gradient $\nabla F(\pi)$ decreases geometrically.
\end{theorem}

\begin{proof}
 First, note that the gradient of $J$ can be written as
 \begin{align*}
  \nabla J(\pi) = \nabla^2 F(\pi) \nabla F(\pi).
 \end{align*}
 Therefore we notice that,
   \begin{align}
    \inner{\tilde{\nabla} F(\pi)}{\nabla J(\pi)}
 &= \inner{\tilde{\nabla} F(\pi)}{\nabla^2 F(\pi) \nabla F(\pi)}\notag\\
    &=\inner{
 \begin{pmatrix}
  -\nabla_{\theta} F(\theta,\omega) \\ \nabla_{\omega} F(\theta,\omega)
 \end{pmatrix} }{
    \begin{pmatrix}
     \nabla_\theta^2 F(\theta,\omega) & \nabla_{\theta}\nabla_{\omega}
     F(\theta,\omega)\\
     \nabla_{\omega}\nabla_{\theta} F(\theta,\omega) & \nabla_{\omega}^2  F(\theta,\omega)
    \end{pmatrix}
 \begin{pmatrix}
  \nabla_{\theta} F(\theta,\omega) \\ \nabla_{\omega} F(\theta,\omega)
 \end{pmatrix}
    }\notag\\
&=-\inner{\nabla_{\theta} F(\theta,\omega)}{\nabla_\theta^2
    F(\theta,\omega) \nabla_{\theta} F(\theta,\omega)}
    + \inner{\nabla_{\omega} F(\theta,\omega)}{\nabla_{\omega}^2
    F(\theta,\omega)\nabla_{\omega} F(\theta,\omega)}\notag\\
&\leq-\delta \left(\|\nabla_{\theta} F(\theta,\omega)\|_2^2
    +\|\nabla_{\omega} F(\theta,\omega)\|_2^2\right) = -\delta \|\nabla F(\pi)\|_2^2
   \label{eq:sufficientinnerp}
   \end{align}
 In other words, Algorithm~\ref{alg:method2} decreases $J$ by an amount
 proportional to the squared norm of $\nabla F(\pi)$.

Now combining the smoothness \eqref{eq:smoothness} with Algorithm \ref{alg:method2},
 we get
\begin{align*}
 J(\pi^{t+1}) &\leq J(\pi^{t}) + \eta\inner{\nabla
 J(\pi^{t})}{\tilde{\nabla}F(\pi^{t})}+ \frac{L\eta^2}{2}\|\tilde{\nabla}
 F(\pi^{t})\|_2^2\\
 &\leq\left(1 -\delta\eta + \frac{L\eta^2}{2}  \right)J(\pi^{t})\\
 &=\left(1 - \frac{\delta^2}{2L}\right)J(\pi^{t}),
\end{align*}
 where we used sufficient decrease \eqref{eq:sufficientinnerp} and 
 $J(\pi)=\|\nabla F(\pi)\|_2^2 =
 \|\tilde{\nabla} F(\pi)\|_2^2$ in the second inequality, and the final equality follows by taking $\eta=\delta/L$.
\end{proof}


\section{Related Algorithms}\label{SecAlgorithmsRelated}

Due to recent interest in GAN type models, there have been attempts to derive other divergence measures and algorithms.
In particular, an alternative Jensen-Shannon divergence has been derived in~\cite{huszar2015generativemodel} and a heuristic algorithm that behaves similarly to the one resulting from this new divergence has been proposed in~\cite{huszar2016reversekl}.

In this section we summarise (some of) the current algorithms and show how
they are related. Note that some algorithms use heuristics that do not
correspond to saddle point optimisation, that is, in the corresponding
maximization and minimization steps they optimise alternative objectives that
do not add up to a coherent joint objective. We include a short discussion of
\cite{Gutmann10nce} because it can be viewed as a special case of GAN.

To illustrate how the discussed algorithms work, we define the objective function 
\begin{align}\label{EqnGAN-Fab}
F(\theta,\omega; \alpha, \beta) 
  =&
  E_{x \sim P}[\log D_{\omega}(x)]
  + \alpha E_{x \sim Q_{\theta}}[\log(1-D_{\omega}(x))]
  - \beta E_{x \sim Q_{\theta}}[\log(D_{\omega}(x))], 
\end{align}
where we introduce two scalar parameters, $\alpha$ and $\beta$, to help us highlight the differences between the algorithms shown in Table~\ref{TableGANalg}. 

\begin{table}[h]
\begin{center}
\begin{tabular}{lcc}
\toprule
Algorithm & Maximisation in $\omega$ & Minimisation in $\theta$
\\
\midrule
NCE~\cite{Gutmann10nce} & $\alpha=1, \beta = 0$ &  NA  \\
GAN-1~\cite{goodfellow2014generativeadversarial} & $\alpha=1, \beta = 0$ & $\alpha = 1, \beta = 0$\\
GAN-2~\cite{goodfellow2014generativeadversarial} & $\alpha=1, \beta = 0$ & $\alpha = 0, \beta = 1$\\
GAN-3~\cite{huszar2016reversekl} & $\alpha=1, \beta = 0$ & $\alpha = 1, \beta = 1$
\\
\bottomrule
\end{tabular}
\end{center}  
\caption{Optimisation algorithms for the GAN objective~\eqref{EqnGAN-Fab}.}
\label{TableGANalg}
\end{table}

\subsubsection*{Noise-Contrastive Estimation (NCE)} 

NCE~\cite{Gutmann10nce} is a method that estimates the parameters of an unnormalised model $p(x ; \omega)$  by performing 
non-linear logistic regression to discriminate between the model and artificially generated noise. To achieve this NCE casts the estimation problem as a ML estimation in a binary classification model where the data is augmented with artificially generated data. The ``true" data items are labeled as positives while the  artificially generated data items are labeled as negatives. The discriminant function is defined as $D_{\omega}(x) = p(x ; \omega) /(p(x ; \omega) + q(x)) $ where $q(x)$ denotes the distribution of the artificially generated data, typically a Gaussian  parameterised by the empirical mean and covariance of the true data. ML estimation in this binary classification model results in an objective that has the form \eqref{EqnGAN-Fab} with $\alpha =1$ amd $\beta = 0$, where the expectations are taken w.r.t. the empirical distribution of augmented data. As a result, NCE can be viewed as a special case of GAN where the generator is fixed and we only have maximise the objective w.r.t. the parameters of the discriminator. Another slight difference is that in this case the data distribution is learned through the discriminator not the generator, however, the method has many conceptual similarities to GAN.

\subsubsection*{GAN-1 and GAN-2} 
The first algorithm (GAN-1) proposed in \cite{goodfellow2014generativeadversarial} performs a stochastic gradient ascent-descent on the objective with $\alpha=1$ and $\beta=0$, however, the authors point out that in practice it is more advantageous to minimise  $-E_{x \sim Q_{\rm \theta}}[\log D_{\omega}(x)]$ instead of $E_{x \sim Q_{\rm \theta}}[\log(1-D_{\omega}(x))]$, we denote this by GAN-2. This is motivated by the observation that in the early stages of training  when $Q_{\theta}$ is not sufficiently well fitted, $D_{\omega}$ can saturate fast leading to weak gradients in $E_{x \sim Q_{\rm \theta}}[\log(1-D_{\omega}(x))]$.  The $-E_{x \sim Q_{\rm \theta}}[\log D_{\omega}(x)]$ term, however, can provide stronger gradients and leads to the same fixed point. This heuristic can be viewed as using $\alpha=1, \beta=0$ in the maximisation step and $\alpha=0, \beta=1$ in the minimisation step\footnote{ A somewhat similar observation regarding the artificially generated data is made in \cite{Gutmann10nce}: in order to have meaningful training one should choose the artificially generated data to be close the the true data, hence the choice of an ML multivariate Gaussian.}. 

\subsubsection*{GAN-3}
In \cite{huszar2016reversekl} a further heuristic for the minimisation step is proposed. Formally, it can be viewed as a combination of the minimisation steps in GAN-1 and GAN-2. In the proposed algorithm, the maximisation step is performed similarly ($\alpha=1, \beta=0$), but the minimisation is done using $\alpha=1$ and $\beta=1$. This choice is motivated by KL optimality arguments. The author makes the observation that the optimal discriminator is
given by 
\begin{equation}
    D^{\ast}(x) = \frac{p(x)}{q_{\theta}(x) + p(x)} 
\end{equation}
and thus, close to optimality,  the minimisation of $  E_{x \sim Q_{\theta}}[\log(1-D_{\omega}(x))]
  - E_{x \sim Q_{\theta}}[\log D_{\omega}(x)] $ corresponds to the minimisation of  the reverse KL divergence $E_{x \sim Q_{\theta}}[\log (q_{\theta}(x)/p(x))]$. This approach can be viewed as choosing $\alpha=1$ and $\beta=1$ in the minimisation step.

\subsubsection*{Remarks on the Weighted Jensen-Shannon Divergence in \cite{huszar2015generativemodel}}

The GAN/variational objective corresponding to alternative Jensen-Shannon divergence measure proposed in \cite{huszar2015generativemodel} (see Jensen-Shannon-weighted in Table 1) is
\begin{align}\label{EqnJSW}
    F(\theta,\omega; \pi) 
  =&
  E_{x \sim P}[\log D_{\omega}(x)]
  -  (1-\pi) E_{x \sim Q_{\theta}}\Big[\log \frac{1-\pi}{1 - \pi D_{\omega}(x)^{1/\pi}}\Big].
\end{align}

Note that we have the $T_{\omega}(x) = \log D_{\omega}(x)$ correspondence. According to the definition of the variational objective, when $T_{\omega}$ is close to optimal then in the minimisation step the objective function is close to the chosen divergence. In this case the optimal discriminator is 
\begin{equation}
  D^{\ast}(x)^{1/\pi} = \frac{p(x)}{ (1-\pi) q_{\theta}(x) + \pi p(x)}. 
\end{equation}
The objective in \eqref{EqnJSW} vanishes when $\pi \in \{0,1 \}$, however,
when $\pi$ is only is close to $0$ and $1$, it can behave similarly to the KL
and reverse KL objectives, respectively. Overall, the connection between GAN-3
and the optimisation of \eqref{EqnJSW} can only be considered as approximate.
To obtain an exact KL or reverse KL behavior one can use the corresponding
variational objectives. For a simple illustration of how these divergences
behave see Section~\ref{sec:gmm} and Section~\ref{SecGMM} below. 

\section{Details of the Univariate Example}\label{SecGMM}
We follow up on the example in Section~\ref{sec:gmm} of the main text by presenting further details about the quality and behavior of the approximations resulting from using various divergence measures. For completeness, we reiterate the setup and then we present further results.


\begin{figure}
    \begin{center}
        \begin{tabular}{cc}
            \resizebox{0.475\textwidth}{!}{\includegraphics{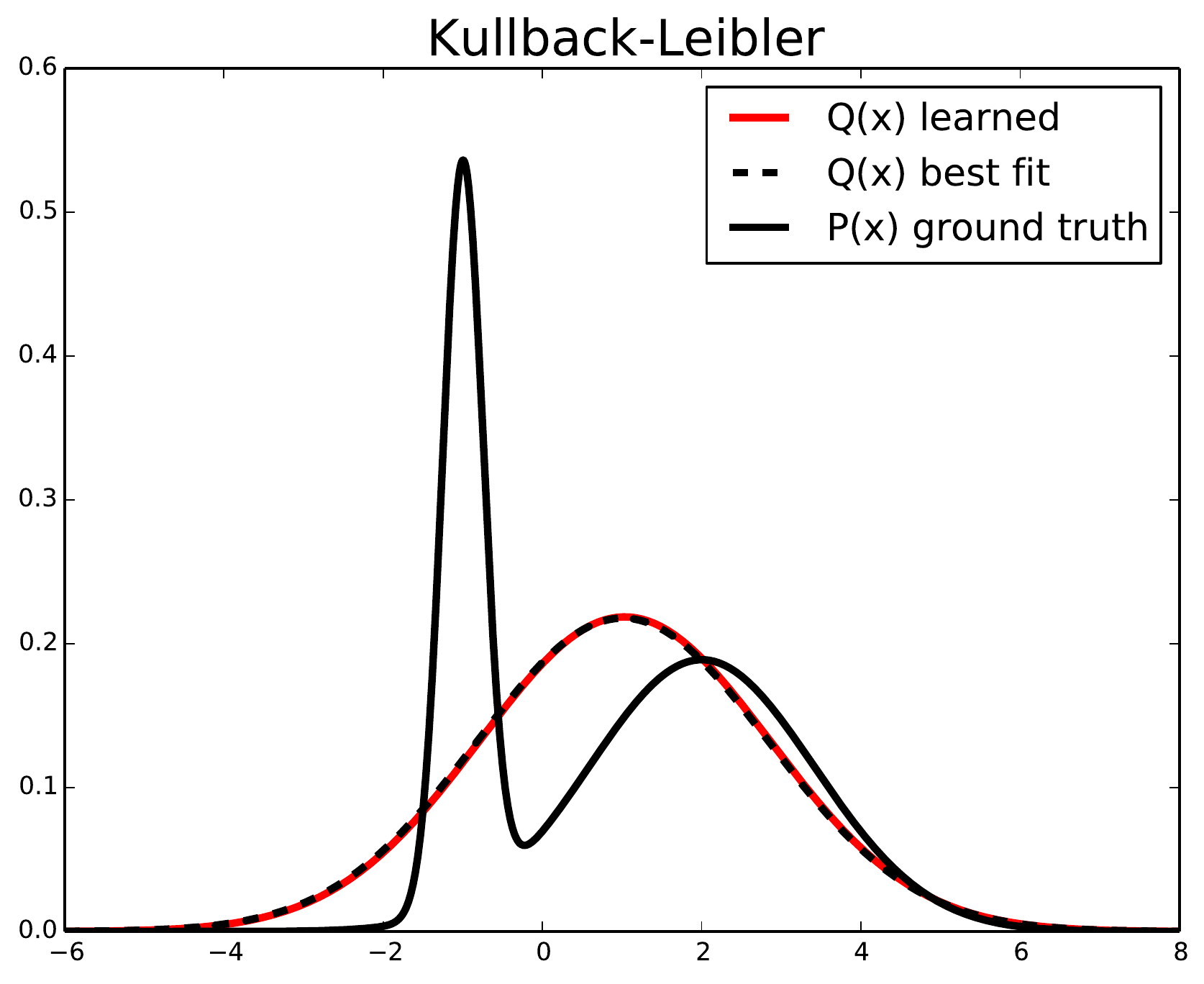}}
            &
            \resizebox{0.475\textwidth}{!}{\includegraphics{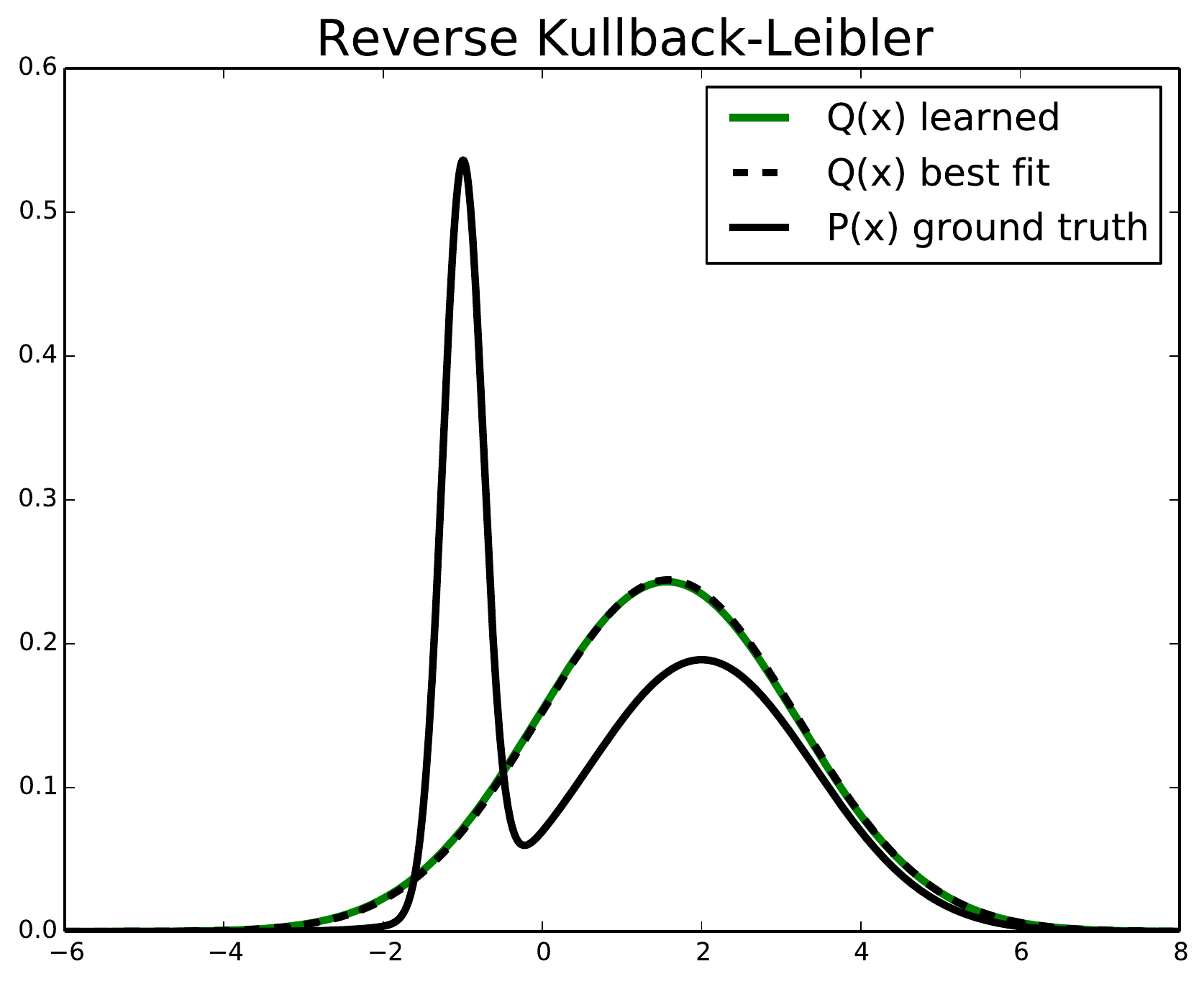}}
            \\
            \resizebox{0.475\textwidth}{!}{\includegraphics{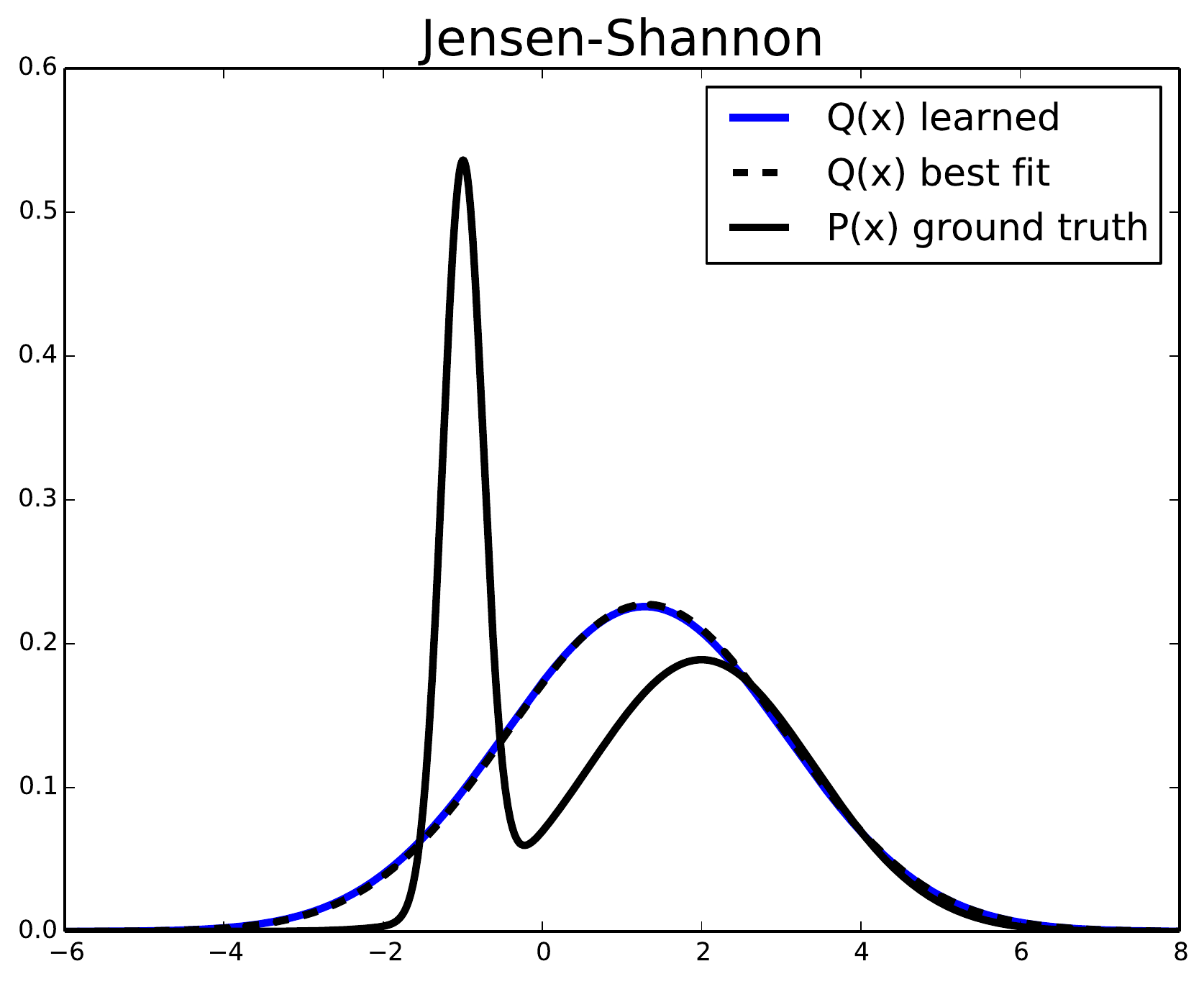}}
            &
            \resizebox{0.475\textwidth}{!}{\includegraphics{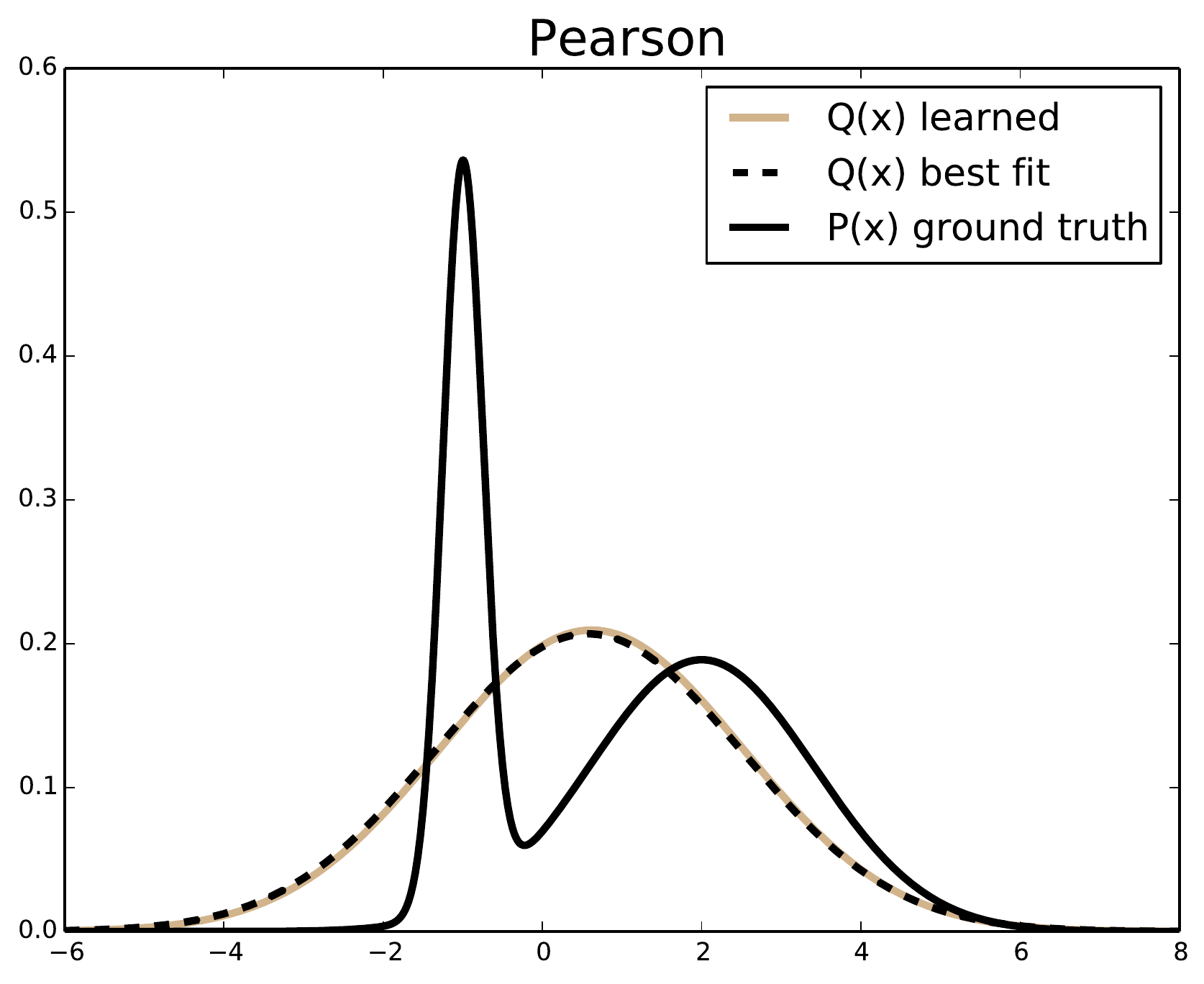}}
            \\  
            \resizebox{0.475\textwidth}{!}{\includegraphics{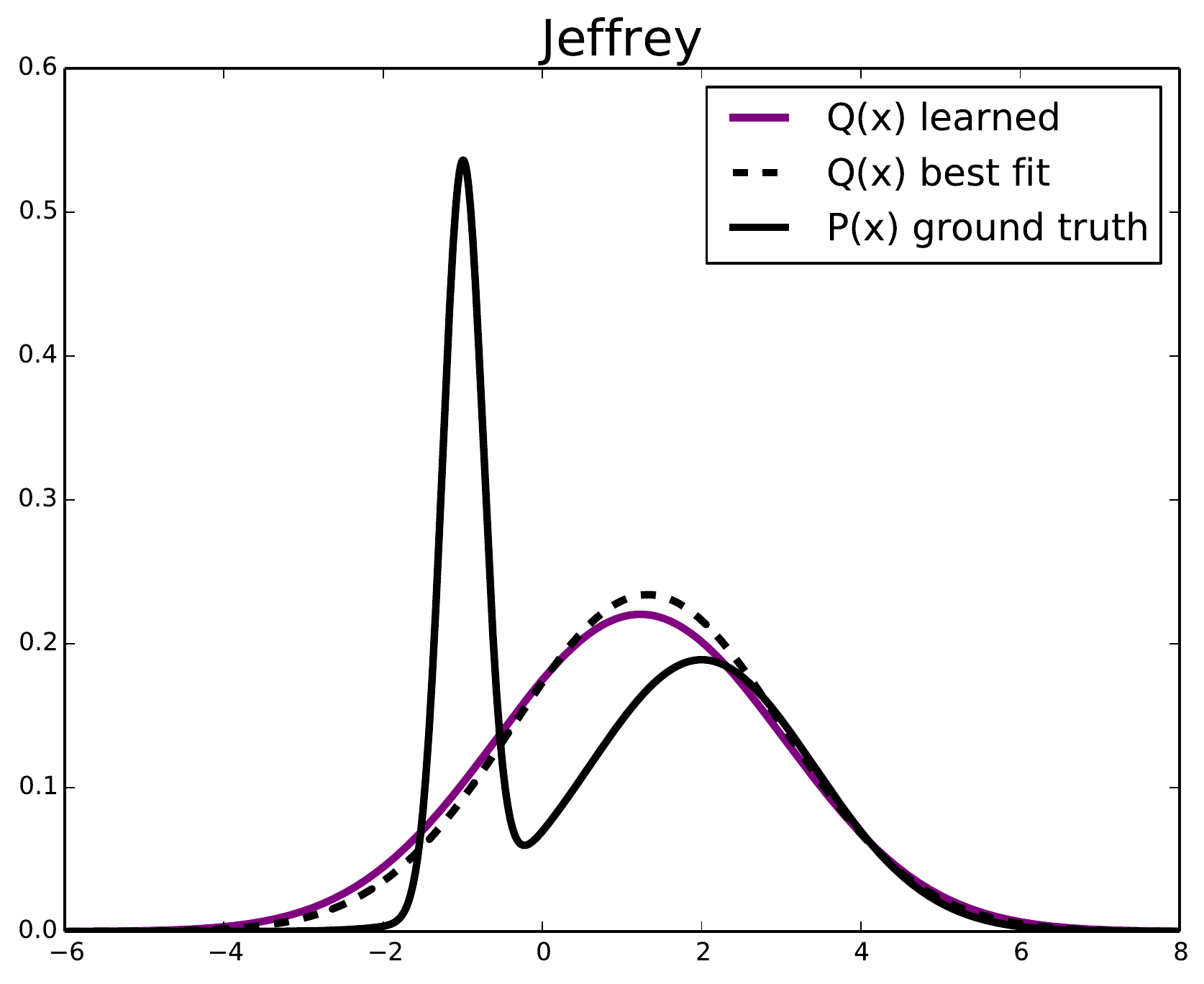}}
            &
            \resizebox{0.475\textwidth}{!}{\includegraphics{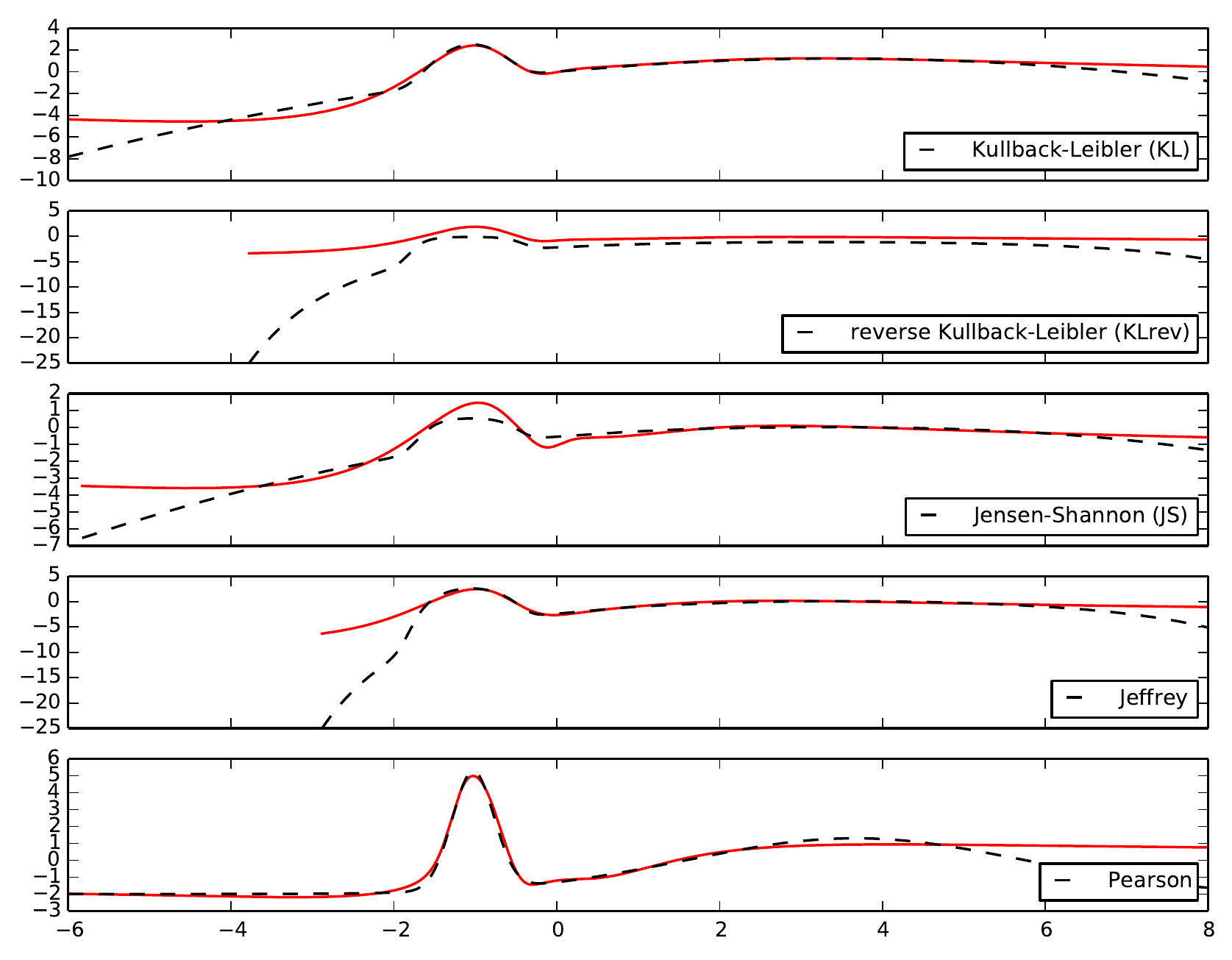}}
        \end{tabular}
    \end{center}
    \caption{\small Gaussian approximation of a mixture of Gaussians. Gaussian approximations obtained by direct optimisation of $D_{f} (p \rvert\lvert q_{\theta^{\ast}})$ (dashed-black) and the optimisation of $F(\hat{\omega}, \hat{\theta})$ (solid-colored). Right-bottom:  optimal variational functions $T^{\ast}$ (dashed) and $T_{\hat{\omega}}$ (solid-red).}
    \label{FigGMM}
\end{figure}

\textbf{Setup.}
We approximate a mixture of Gaussian \footnote{The plots on Figure~\ref{FigGMM} correspond to $p(x) = (1-w) N(x; m_1, v_1) + w N(x; m_2, v_2)$ with $w = 0.67, m_1 =-1, v_1 = 0.0625, m_2 =2, v_2 = 2$.} by learning a Gaussian distribution.
The model $Q_{\theta}$ is represented by a linear function which receives a random $z \sim \mathcal{N}(0,1)$ and outputs
\begin{equation}
  G_{\theta}(z) = \mu + \sigma z,
\end{equation}
where $\theta = (\mu, \sigma)$ are the parameters to be learned.
For the variational function $T_{\omega}$ we use the neural network 
\begin{eqnarray}
& x &
  \to \textrm{Linear(1,64)}
  \to \textrm{Tanh}
  \to \textrm{Linear(64,64)}
  \to \textrm{Tanh}
  \to \textrm{Linear(64,1)}.
\end{eqnarray}
We optimise the objective $F(\omega, \theta)$ by using the single-step
gradient method presented in Section~\ref{sec:SSG} of the main text .
In each step we sample batches of size $1024$ each for both $p(x)$ and $p(z)$
and we use a step-size of $0.01$ for updating both $\omega$ and $\theta$.
We compare the results to the best fit provided by the exact optimisation
of $D_f(P \| Q_{\theta})$ w.r.t. $\theta$, which is feasible in this case by
solving the required integrals numerically. 
We use $(\hat{\omega}, \hat{\theta})$ (learned) and $\theta^{\ast}$ (best fit)  to distinguish the parameters sets used in these
two approaches.  

\textbf{Results.}
The panels in Figure~\ref{FigGMM} shows the density function of the data
distribution  as well as the Gaussian approximations corresponding to a few $f$-divergences form Table~\ref{tab:f-divergences}.
As expected, the KL approximation covers the data
distribution by fitting its mean and variance while KL-rev has more of a
mode-seeking behavior~\cite{minka2005divergence}. The fit corresponding to the Jensen-Shannon divergence is somewhere between
KL and KL-rev.
All Gaussian approximations resulting from neural network training are
close to the ones obtained by direct optimisation of the divergence (learned vs. best fit).

In the right--bottom panel of Figure~\ref{FigGMM} we compare the variational functions
$T_{\hat{\omega}}$ and $T^{\ast}$. The latter is defined as $T^{\ast}(x) = f^{\prime}(p(x)/ q_{{\theta}^{\ast}}(x))$, see main text.
The objective value corresponding to $T^{\ast}$ is the true
divergence $D_{f}(P \rvert\lvert Q_{{\theta}^{\ast}})$.
In the majority of the cases our $T_{\hat{\omega}}$ is close to $T^{\ast}$
in the area of interest.
The discrepancies around the tails are due to the fact that
(1) the class of functions resulting from the $\rm tanh$ activation function has limited capability representing the tails, and
(2) in the Gaussian case there is a lack of data in the tails. 
These limitations, however, do not have a significant effect on the learned parameters.

\section{Details of the Experiments}\label{SecExpNN}
In this section we present the technical setup as well as the architectures we
used in the experiments described in Section~\ref{sec:experiments}.

\subsection{Deep Learning Environment}
We use the deep learning framework \emph{Chainer}~\cite{tokui2015chainer},
version 1.8.1, running on CUDA 7.5 with CuDNN v5 on NVIDIA GTX TITAN X.

\subsection{MNIST Setup}
\paragraph{MNIST Generator}
\begin{eqnarray}
& z &
	\to \textrm{Linear(100, 1200)}
	\to \textrm{BN}
	\to \textrm{ReLU}
	\to \textrm{Linear(1200, 1200)}
	\to \textrm{BN}
	\to \textrm{ReLU}
\nonumber\\
& &
	\to \textrm{Linear(1200, 784)}
	\to \textrm{Sigmoid}
\end{eqnarray}
All weights are initialized at a weight scale of $0.05$, as
in~\cite{goodfellow2014generativeadversarial}.

\paragraph{MNIST Variational Function}
\begin{eqnarray}
& x &
	\to \textrm{Linear(784,240)}
	\to \textrm{ELU}
	\to \textrm{Linear(240,240)}
	\to \textrm{ELU}
	\to \textrm{Linear(240,1)},
\end{eqnarray}
where $\textrm{ELU}$ is the exponential linear unit~\cite{clevert2015elu}.
All weights are initialized at a weight scale of $0.005$, one order of
magnitude smaller than in~\cite{goodfellow2014generativeadversarial}.

\paragraph{Variational Autoencoders}
For the variational autoencoders~\cite{kingma2013vae}, we 
used the example implementation included with
\emph{Chainer}~\cite{tokui2015chainer}. We trained
for 100 epochs with 20 latent dimensions.

\subsection{LSUN Natural Images}

\begin{eqnarray}
& z &
	\to \textrm{Linear(100, $6 \cdot 6 \cdot 512$)}
	\to \textrm{BN}
	\to \textrm{ReLU}
	\to \textrm{Reshape(512,6,6)}
\nonumber\\
& &
	\to \textrm{Deconv(512,256)}
	\to \textrm{BN}
	\to \textrm{ReLU}
	\to \textrm{Deconv(256,128)}
	\to \textrm{BN}
	\to \textrm{ReLU}
\nonumber\\
& &
	\to \textrm{Deconv(128,64)}
	\to \textrm{BN}
	\to \textrm{ReLU}
	\to \textrm{Deconv(64,3)},
\end{eqnarray}
where all $\textrm{Deconv}$ operations use a kernel size of four and a stride
of two.

\end{appendices}

\end{document}